\documentclass{article}

    \PassOptionsToPackage{numbers, compress}{natbib}
\usepackage{amsthm}
\usepackage{amsmath,amssymb,amsfonts}
\usepackage{algorithmic}
\usepackage{graphicx}
\usepackage{textcomp}
\usepackage{xcolor}
\usepackage{algorithm}

\DeclareMathOperator*{\argmax}{arg\,max}

\usepackage{subfigure}
\usepackage{caption}
\usepackage{booktabs}
\usepackage{multirow}
\newcommand\Tstrut{\rule{0pt}{2.2ex}}

\usepackage{wrapfig} % for wrapping tables/figures
\usepackage{float}

    \usepackage[final]{neurips_2022}

\usepackage[utf8]{inputenc} % allow utf-8 input
\usepackage[T1]{fontenc}    % use 8-bit T1 fonts
\usepackage{hyperref}       % hyperlinks
\usepackage{url}            % simple URL typesetting
\usepackage{booktabs}       % professional-quality tables
\usepackage{amsfonts}       % blackboard math symbols
\usepackage{nicefrac}       % compact symbols for 1/2, etc.
\usepackage{microtype}      % microtypography
\usepackage{xcolor}         % colors

\title{A Theoretical Study on Solving Continual Learning} 

\author{%
Gyuhak Kim\thanks{Equal contribution}$^{\;\; 1}$, Changnan Xiao$^{* 2}$, Tatsuya Konishi\thanks{The work was done when this author was visiting Bing Liu's group at University of Illinois at Chicago}~$\;^{3}$, Zixuan Ke$^{1}$, Bing Liu\thanks{Correspondance author. Bing Liu <\texttt{liub@uic.edu}>}$^{\;\; 1}$ \\
$^{1}$ University of Illinois at Chicago \\
$^{2}$ ByteDance \\
$^{3}$ KDDI Research \\
}

\begin{document}
\maketitle

\theoremstyle{plain}
\newtheorem{theorem}{Theorem}
\newtheorem{proposition}{Proposition}
\newtheorem{lemma}{Lemma}
\newtheorem{corollary}{Corollary}
\theoremstyle{definition}
\newtheorem{definition}{Definition}
\newtheorem{assumption}{Assumption}
\theoremstyle{remark}
\newtheorem{remark}{Remark}

\begin{abstract}
Continual learning (CL) learns a sequence of tasks incrementally.~There are two popular CL settings, \textit{class incremental learning} (CIL) and \textit{task incremental learning} (TIL). A major challenge of CL is \textit{catastrophic forgetting} (CF). While several techniques are available to effectively overcome CF for TIL, CIL remains to be challenging 
due to the additional difficulty of \textit{inter-task class separation}.~So far little theoretical work has been done to provide a \textit{principled guidance} 
and \textit{necessary and sufficient} conditions
for solving the CIL problem.~This paper performs such a study.~It first
probabilistically decomposes the CIL problem
into two sub-problems:~\textit{within-task prediction} (WP) and \textit{task-id prediction} (TP). It further proves that TP is correlated with \textit{out-of-distribution} (OOD) detection. 
{\color{black}The key \textit{result} 
is that regardless of whether WP and TP or OOD detection are defined explicitly or implicitly by a CIL algorithm, good WP and good
TP or OOD detection are \textit{necessary} and \textit{sufficient} for good CIL performances. Additionally, TIL is simply WP.}
Based on the theoretical result, new CIL methods are also designed, which outperform strong baselines {\color{black}in both CIL and TIL} settings by a large margin.\footnote{{The code is available at \url{https://github.com/k-gyuhak/WPTP}}}
\end{abstract}

\section{Introduction}
\label{sec-intro}
Continual learning aims to incrementally learn a sequence of tasks~\citep{chen2018lifelong}. Each task consists of a set of classes to be learned together. A major challenge of CL is \textit{catastrophic forgetting} (CF). {\color{black}Although a large number of CL techniques have been proposed, they are mainly empirical. Limited theoretical research has done on how to solve CL. This paper performs such a theoretical study about the necessary and sufficient conditions for effective CL. There are two main CL settings that have been extensively studied:} \textit{class incremental learning} (CIL) and \textit{task incremental learning} (TIL)~\citep{van2019three}. In CIL, the learning process builds a single classifier for all tasks/classes 
learned so far. In testing, a test instance from any class may be presented for the model to classify. No prior task information (e.g., task-id) of the test instance is provided. Formally, CIL is defined as follows.  

\textbf{Class incremental learning} (CIL). CIL learns a sequence of tasks, $1, 2, ..., T$. Each task $k$ has a training  dataset
$\mathcal{D}_k=\{(x_k^i, y_k^i)_{i=1}^{n_k}\}$, where $n_k$ is the number of data samples in task $k$, and $x_k^i \in \mathbf{X}$ is an input sample and $y_k^i \in \mathbf{Y}_k$ (the set of all classes of task $k$) is its class label. All $\mathbf{Y}_k$'s are disjoint ($\mathbf{Y}_k \cap \mathbf{Y}_{k'} = \emptyset,\, \forall k \neq k'$) and $\bigcup_{k=1}^T \mathbf{Y}_k = \mathbf{Y}$. 
The goal of CIL is to construct a single predictive function or classifier $f : \mathbf{X} \rightarrow \mathbf{Y}$ that can identify the class label $y$ of each given test instance $x$.

In the TIL setup, each task is a separate classification problem. {\color{black}For example, one task could be to classify different breeds of dogs and another task could be to classify different types of animals (the tasks may not be disjoint).} One model is built for each task in a shared network. In testing, the task-id of each test instance is provided and the system uses only the specific model for the task (dog or animal classification) to classify the test instance. Formally, TIL is defined as follows.

\textbf{Task incremental learning} (TIL). TIL learns a sequence of tasks, $1, 2, ..., T$. Each task $k$ has a training dataset
$\mathcal{D}_k=\{((x_k^i, k), y_k^i)_{i=1}^{n_k}\}$, 
where $n_k$ is the number of data samples in task $k \in \mathbf{T} = \{1, 2, ..., T\}$, and $x_k^i \in \mathbf{X}$ is an input sample and $y_k^i \in \mathbf{Y}_k \subset \mathbf{Y}$ is its class label.
The goal of TIL is to construct a predictor $f: \mathbf{X} \times \mathbf{T} \rightarrow \mathbf{Y}$ to identify the class label $y \in \mathbf{Y}_k$ for $(x, k)$ (the given test instance $x$ from task $k$).

Several techniques are available to effectively overcome CF for TIL (with almost no CF)~\cite{Serra2018overcoming,supsup2020}. However, CIL remains to be highly challenging due to the additional problem of \textit{Inter-task Class Separation} (ICS) (establishing decision boundaries between the classes of the new task and the classes of the previous tasks) because the previous task data are not accessible.
Before discussing the proposed work, we recall 
the \textit{closed-world} assumption made by traditional machine learning, i.e., \textit{the classes seen in testing must have been seen in training}~\cite{chen2018lifelong,liu2021self}. However, in many applications, there are unknowns in testing, which is called the \textit{open world} setting~\cite{chen2018lifelong,liu2021self}. In open world learning, the training (or known) classes are called \textit{in-distribution} (IND) classes. A classifier built for the open world can (1) classify test instances of training/IND classes to their respective classes, which is called \textit{IND prediction}, and (2) detect test instances that do not belong to any of the IND/known classes but some unknown or \textit{out-of-distribution} (OOD) classes, which is called \textit{OOD detection}. Many OOD detection algorithms can perform both IND prediction and OOD detection~\cite{tack2020csi,liang2018enhancing,esmaeilpour2022zero,wang2022omg}. {\color{black}The commonality of OOD detection and CL is that they both need to consider future unknowns.}

This paper conducts a theoretical study of CIL, which is applicable to any CIL classification model. Instead of focusing on the traditional PAC generalization bound~\cite{pentina2014pac,karakida2022learning}, we focus on how to solve the CIL problem.
We first decompose the CIL problem into two sub-problems in a probabilistic framework: \textit{Within-task Prediction} (WP) and \textit{Task-id Prediction} (TP). WP means that the prediction for a test instance is only done within the classes of the task to which the test instance belongs, {\color{black}which is basically the TIL problem.} TP predicts the task-id. TP is needed because in CIL, task-id is not provided in testing. This paper then proves based on the popular cross-entropy loss that (1) the CIL performance is bounded by WP and TP performances, and (2) TP and task OOD detection performance bound each other (which connects CL and OOD detection). {\color{black}The key result is that regardless of whether WP and TP or OOD detection are defined explicitly or implicitly by a CIL algorithm, good WP and good TP or OOD detection are
\textit{necessary} and \textit{sufficient} conditions for good CIL performances. This result is applicable to both batch/offline and online CIL and to CIL problems with blurry task boundaries. {\color{black}The intuition is also quite simple because if a CIL model is perfect at detecting OOD samples for each task (which solves the ICS problem), then CIL is reduced to WP.} 
} 

This theoretical result provides a principled guidance for solving the CIL problem, i.e., to help design better CIL algorithms that can achieve strong WP and TP performances.
Since WP is basically IND prediction for each task and most OOD techniques perform both IND prediction and OOD detection, 
to achieve good CIL accuracy, a strong OOD performance for each task is necessary.

Based on the theoretical guidance, several new CIL methods are designed, including techniques based on the integration of a TIL method and an OOD detection method for CIL, which outperform strong baselines {\color{black}in both the CIL and TIL settings by a large margin}. This combination is particularly attractive because TIL has achieved no forgetting, and we only need a strong OOD technique that can perform both IND prediction and OOD detection to learn each task to achieve strong CIL results.

\section{Related Work}
\label{sec.related}

Although numerous CL techniques have been proposed, little study has been done to provide a theoretical guidance on how to solve the problem. Existing approaches mainly belong to several 
categories.  
Using regularization~\cite{kirkpatrick2017overcoming,Li2016LwF} to minimize changes to model parameters learned from previous tasks is a popular approach~\cite{Jung2016less,Camoriano2017incremental,zenke2017continual,ritter2018online,schwarz2018progress,xu2018reinforced,castro2018end,hu2019overcoming,Dhar2019CVPR,lee2019overcoming,ahn2019neurIPS,Liu2020,Zhu_2021_CVPR_pass}. 
Memorizing some old examples and using them to jointly train the % adjust the old models in learning a 
new task is another popular approach (called \textit{replay})~\cite{Rusu2016,Lopez2017gradient,Rebuffi2017,Chaudhry2019ICLR,hou2019learning,wu2019large,rolnick2019neurIPS, NEURIPS2020_b704ea2c_derpp,rajasegaran2020adaptive,Liu2020AANets,Cha_2021_ICCV_co2l,yan2021dynamically,wang2022memory,guo2022online}. 
Some systems learn to generate pseudo training data of old tasks and use them to jointly train the new task, called \textit{pseudo-replay}~\cite{Gepperth2016bio,Kamra2017deep,Shin2017continual,wu2018memory,Seff2017continual,Kemker2018fearnet,hu2019overcoming,Rostami2019ijcai,ostapenko2019learning}. Orthogonal projection learns each task in an orthogonal space to other tasks % and thus have minimum interference
\cite{zeng2019continuous,guo2022adaptive,chaudhry2020continual}. Our theoretical study is applicable to any continually trained classification models. % We try some representative methods from the categories to demonstrate our idea in our experiments.

{\textit{Parameter isolation} is yet another popular approach, which makes different subsets (which may overlap) of the network parameters dedicated to different tasks using masks~\cite{Serra2018overcoming,ke2020continual,Mallya2017packnet, supsup2020,NEURIPS2019_3b220b43compact}. 
This approach is particularly suited for TIL. Several methods have almost completely overcome forgetting. HAT~\cite{Serra2018overcoming} and CAT~\cite{ke2020continual} protect previous tasks by masking the important parameters to those tasks. PackNet~\cite{Mallya2017packnet}, CPG~\cite{NEURIPS2019_3b220b43compact} and SupSup~\cite{supsup2020} find an isolated sub-network for each task. 
HyperNet~\cite{von2019continual} initializes task-specific parameters conditioned on task-id. {ADP~\cite{Yoon2020Scalable} decomposes parameters into shared and adaptive parts to construct an order robust TIL system.} CCLL~\cite{singh2020calibrating} uses task-adaptive calibration in convolution layers.
Our methods designed based on the proposed theory make use of two parameter isolation-based TIL methods and two OOD detection methods. A strong OOD detection method CSI in~\cite{tack2020csi} helps produce very strong CIL results. CSI is based on data augmentation~\cite{he2020momentum} and contrastive learning~\cite{chen2020simple}. Excellent surveys of OOD detection include~\cite{bulusu2020anomalous,geng2020recent}.
} 

Some methods have used a TIL method for CIL with an additional task-id prediction technique. iTAML~\cite{rajasegaran2020itaml}
requires each test batch to be from a single task. This is not practical as test samples usually come one by one. CCG~\cite{abati2020conditional} builds a separate network to predict the task-id. Expert Gate~\cite{Aljundi2016expert} constructs a separate autoencoder for each task. {HyperNet~\cite{von2019continual} and PR-Ent~\cite{henning2021posterior} use entropy to predict the task id.}
Since none of these papers is a theoretical study, they did not know that strong OOD detection is the key. Our methods based on OOD detection perform dramatically better.

Several theoretical studies have been made on lifelong/continual learning. However, they focus on traditional generalization bound. \cite{pentina2014pac} proposes a PAC-Bayesian framework to provide a learning bound on expected error in future tasks by the average loss on the observed tasks. The work in \cite{lee2021continual} studies the generalization error by task similarity and \cite{karakida2022learning} studies the dependence of generalization error on sample size or number of tasks including forward and backward transfer.
\cite{bennani2020generalisation} shows that orthogonal gradient descent gives a tighter generalization bound than SGD.
Our work is very different as we focus on how to solve the CIL problem, which is orthogonal to the existing theoretical analysis.

\section{CIL by Within-Task Prediction and Task-ID Prediction}\label{sec.theorem}
{\color{black}This section presents our theory. It first shows that the CIL performance improves if the within-task prediction (WP) performance and/or the task-id prediction (TP) performance improve, and then shows that TP and OOD detection bound each other, which indicates that CIL performance is controlled by WP and OOD detection. This connects CL and OOD detection. Finally, we study the necessary conditions for a good CIL model, which includes a good WP, and a good TP (or OOD detection). 
}

\subsection{CIL Problem Decomposition} \label{sec:decomposition}
This sub-section first presents the assumptions made by CIL based on its definition and then proposes a decomposition of the CIL problem into two sub-problems. A CL system learns a sequence of tasks $\{(\mathbf{X}_k, \mathbf{Y}_k)\}_{k=1,\dots,T}$, where $\mathbf{X}_{k}$ is the domain of task $k$ and $\mathbf{Y}_k$ are classes of task $k$ as $\mathbf{Y}_k = \{\mathbf{Y}_{k, j}\}$, 
where $j$ indicates the $j$th class in task $k$. 
Let $\mathbf{X}_{k, j}$ to be the domain of $j$th class of task $k$, where $\mathbf{X}_{k} = \bigcup_j \mathbf{X}_{k, j}$. 
{\color{black}For accuracy, we will use $x \in \mathbf{X}_{k, j}$ instead of $\mathbf{Y}_{k, j}$ in probabilistic analysis.} 
Based on the definition of class incremental learning (CIL) (Sec.~\ref{sec-intro}), the following assumptions are implied,
\begin{assumption}
    The domains of classes of the same task are disjoint, i.e., $\mathbf{X}_{k, j} \cap \mathbf{X}_{k, j'} = \emptyset,\, \forall j \neq j'$.
\end{assumption}
\begin{assumption}
    The domains of tasks are disjoint, i.e., $\mathbf{X}_k \cap \mathbf{X}_{k'} = \emptyset,\, \forall k \neq k'$.
\end{assumption}
For any ground event $D$, the goal of a CIL problem is to learn $\mathbf{P}(x \in \mathbf{X}_{k, j} | D)$. This can be decomposed into two probabilities, \textit{within-task IND prediction} (WP) probability and \textit{task-id prediction} (TP) probability. WP probability is
$\mathbf{P} (x \in \mathbf{X}_{k, j} | x \in \mathbf{X}_{k}, D)$ and TP probability is $\mathbf{P}(x \in \mathbf{X}_k | D)$.
We can rewrite the CIL problem using WP and TP based on the two assumptions, 
\begin{align}
     \mathbf{P}(x \in \mathbf{X}_{k_0, j_0} | D)
     &= \sum_{k=1,\dots,n} \mathbf{P} (x \in \mathbf{X}_{k, j_0} | x \in \mathbf{X}_k, D) \mathbf{P}(x \in \mathbf{X}_k | D) \label{eq:prob_sum} \\
     &= \mathbf{P} (x \in \mathbf{X}_{k_0, j_0} | x \in \mathbf{X}_{k_0}, D) \mathbf{P}(x \in \mathbf{X}_{k_0} | D) \label{eq:cil_in_til_and_tp}
\end{align}
where $k_0$ means a particular task and $j_0$ a particular class in the task. 

{\color{black}Some remarks are in order about Eq.~\ref{eq:cil_in_til_and_tp} and our subsequent analysis to set the stage.} 

\begin{remark}\label{remark1}
Eq.~\ref{eq:cil_in_til_and_tp} shows that if we can improve either the WP or TP performance, or both, we can improve the CIL performance.   
\end{remark}
\begin{remark} \label{remark:alg}
{\color{black}It is important to note that our theory is not concerned with the learning algorithm or the training process, but we will propose some concrete learning algorithms based on the theoretical result in the experiment section. 
}  
\end{remark}
\begin{remark}\label{remark:disjoint}
{\color{black}Note that the CIL definition and the subsequent analysis are applicable to tasks with any number of classes (including only one class per task) and to online CIL where the training data for each task or class comes gradually in a data stream and may also cross task boundaries (blurry tasks~\cite{bang2021rainbow}) because our analysis is based on an already-built CIL model after training. Regarding blurry task boundaries, suppose dataset 1 has classes \{dog, cat, tiger\} and dataset 2 has classes \{dog, computer, car\}. We can define task 1 as \{dog, cat, tiger\} and task 2 as \{computer, car\}. The shared class \textit{dog} in dataset 2 is just additional training data of \textit{dog} appeared after task 1.} 
\end{remark}
\begin{remark}\label{remark:meaningofeq2}
{\color{black}Furthermore, CIL = WP * TP in Eq.~\ref{eq:cil_in_til_and_tp} means that when we have WP and TP (defined either explicitly or implicitly by implementation), we can find a corresponding CIL model defined by WP * TP. Similarly, when we have a CIL model, we can find the corresponding underlying WP and TP defined by their probabilistic definitions.}
\end{remark}
In the following sub-sections, we develop this further concretely to derive the sufficient and necessary conditions for solving the CIL problem in the context of cross-entropy loss as it is used in almost all supervised CL systems.

\subsection{CIL Improves as WP and/or TP Improve}\label{sec:cil_improve_by_til_and_tp}

{\color{black}As stated in Remark~\ref{remark:alg} above,} the study here is based on a \textit{trained CIL model} and not concerned with the algorithm used in training the model. We use cross-entropy as the performance measure of a trained model as it is the most popular loss function used in supervised CL. For experimental evaluation, we use \textit{accuracy} following CL papers. Denote the cross-entropy of two probability distributions $p$ and $q$ as 
\begin{align}
    H(p, q) \overset{def}{=} - \mathbb{E}_p [\log q] = - \sum_i p_i \log q_i.
\end{align}
For any $x \in \mathbf{X}$, let $y$ to be the CIL ground truth label of $x$, where $y_{k_0, j_0} = 1$ if $x \in \mathbf{X}_{k_0, j_0}$ otherwise $y_{k, j} = 0$, $\forall (k, j) \neq (k_0, j_0)$.
Let $\Tilde{y}$ be the WP ground truth label of $x$, where $\Tilde{y}_{k_0, j_0} = 1$ if $x \in \mathbf{X}_{k_0, j_0}$ otherwise $\Tilde{y}_{k_0, j} = 0$, $\forall j \neq j_0$.
Let $\Bar{y}$ be the TP ground truth label of $x$, where $\Bar{y}_{k_0} = 1$ if $x \in \mathbf{X}_{k_0}$ otherwise $\Bar{y}_k = 0$, $\forall k \neq k_0$.
Denote
\begin{align}
    H_{WP} (x) &= H(\Tilde{y}, \{\mathbf{P}(x \in \mathbf{X}_{k_0, j} | x \in \mathbf{X}_{k_0}, D)\}_{j}), \\
    H_{CIL} (x) &= H(y, \{\mathbf{P}(x \in \mathbf{X}_{k, j} | D)\}_{k, j}), \\
    H_{TP} (x) &= H(\Bar{y}, \{\mathbf{P}(x \in \mathbf{X}_k | D)\}_{k})
\end{align}
where $H_{WP}$, $H_{CIL}$ and $H_{TP}$ are the cross-entropy values of WP, CIL and TP, respectively.
We now present our first theorem. The theorem connects CIL to WP and TP, and suggests that by having a good WP or TP, the CIL performance improves as the upper bound for the CIL loss decreases.
\begin{theorem}
\label{thm:ce}
If $H_{TP}(x) \leq \delta$ and $H_{WP}(x) \leq \epsilon$, we have
$
H_{CIL} (x) \leq \epsilon + \delta.
$ 
\end{theorem}

The detailed proof is given in Appendix~\ref{prf:ce}. This theorem  holds regardless of whether WP and TP are trained together or separately.
When they are trained separately, if WP is fixed and we let  $\epsilon=H_{WP}(x)$,  $H_{CIL}(x) \leq H_{WP}(x) + \delta$, which means if TP is better, CIL is better. 
Similarly, if TP is fixed, we have $H_{CIL}(x) \leq \epsilon + H_{TP} (x)$.
When they are trained concurrently, there exists a functional relationship between $\epsilon$ and $\delta$ depending on implementation. 
But no matter what it is, when $\epsilon + \delta$ decreases, CIL gets better.

{\color{black}Theorem~\ref{thm:ce} holds for any $x \in \mathbf{X} = \bigcup_k \mathbf{X}_k$} that satisfies $H_{TP} (x) \leq \delta$ or $H_{WP} (x) \leq \epsilon$. To measure the overall performance under expectation, we present the following corollary.
\begin{corollary}\label{cor:expectation}
Let $U(\mathbf{X})$ represents the uniform distribution on $\mathbf{X}$. i) If $\mathbb{E}_{x \sim U(\mathbf{X})} [H_{TP}(x)] \leq \delta$, then $\mathbb{E}_{x\sim U(\mathbf{X})} [H_{CIL} (x)] \leq \mathbb{E}_{x \sim U(\mathbf{X})} [H_{WP} (x)] + \delta$. Similarly, ii) $\mathbb{E}_{x \sim U(\mathbf{X})} [H_{WP} (x)] \leq \epsilon$, then $\mathbb{E}_{x\sim U(\mathbf{X})} [H_{CIL} (x)] \leq \epsilon + \mathbb{E}_{x \sim U(\mathbf{X})} [H_{TP}(x)]$.
\end{corollary}
The proof is given in Appendix~\ref{prf:expectation}. The corollary is a direct extension of Theorem~\ref{thm:ce} in expectation. The implication is that given TP performance, CIL is positively related to WP. The better the WP is, the better the CIL is as the upper bound of the CIL loss decreases. Similarly, given WP performance, a better TP performance results in a better CIL performance. Due to the positive relation, we can improve CIL by improving either WP or TP using their respective methods developed in each area.

\subsection{Task Prediction (TP) to OOD Detection}

Building on Eq.~\ref{eq:cil_in_til_and_tp}, we have studied the relationship of CIL, WP and TP in Theorem \ref{thm:ce}. We now connect TP and OOD detection. They are shown to be dominated by each other to a constant factor.

We again use cross-entropy $H$ to measure the performance of TP and OOD detection of a trained network as in Sec.~\ref{sec:cil_improve_by_til_and_tp}
To build the connection between $H_{TP}(x)$ and OOD detection of each task, we first define the notations of OOD detection. We use $\mathbf{P}'_k (x \in \mathbf{X}_k | D)$ to represent the probability distribution predicted by the $k$th task's OOD detector. 
Notice that the task prediction (TP) probability distribution $\mathbf{P} (x \in \mathbf{X}_k | D)$ is a categorical distribution over $T$ tasks, while the OOD detection probability distribution $\mathbf{P}'_k (x \in \mathbf{X}_k | D)$ is a Bernoulli distribution. For any $x \in \mathbf{X}$, define 
\begin{align}
    H_{OOD, k} (x) = \left\{ 
    \begin{aligned}
    H(1, \mathbf{P}'_k (x \in \mathbf{X}_k | D)) = - \log \mathbf{P}'_k (x \in \mathbf{X}_k | D),\ x \in \mathbf{X}_k, \\
    H(0, \mathbf{P}'_k (x \in \mathbf{X}_k | D)) = - \log \mathbf{P}'_k (x \notin \mathbf{X}_k | D),\ x \notin \mathbf{X}_k.
    \end{aligned}
    \right.
\end{align}
In CIL, the OOD detection probability for a task can be defined using the output values corresponding to the classes of the task. Some examples of the function is a sigmoid of maximum logit value or a maximum softmax probability after re-scaling to 0 to 1. 
{It is also possible to define the OOD detector directly as a function of tasks instead of a function of the output values of all classes of tasks, i.e. Mahalanobis distance.} The following theorem shows that TP and OOD detection bound each other.
\begin{theorem}
\label{thm:tp_ood}
i) If $H_{TP} (x) \leq \delta$, let $\mathbf{P}'_k (x \in \mathbf{X}_k | D) = \mathbf{P} (x \in \mathbf{X}_k | D)$, then $H_{OOD, k} (x) \leq \delta, \forall\, k = 1, \dots, T$. ii) If $H_{OOD, k} (x) \leq \delta_k, k=1,\dots,T$, let $\mathbf{P} (x \in \mathbf{X}_k | D) = \frac{\mathbf{P}_k' (x \in \mathbf{X}_k |D)}{\sum_k \mathbf{P}_k' (x \in \mathbf{X}_k |D)}$, then $H_{TP} (x) \leq (\sum_k \mathbf{1}_{x \in \mathbf{X}_k} e^{\delta_{k}}) (\sum_k 1 - e^{-\delta_k})$, where $\mathbf{1}_{x \in \mathbf{X}_k}$ is an indicator function.
\end{theorem}
See Appendix~\ref{prf:tp_ood} for the proof. 
As we use cross-entropy, the lower the bound, the better the performance is. The first statement (i) says that the OOD detection performance improves if the TP performance gets better (i.e., lower $\delta$). Similarly, the second statement (ii) says that the TP performance improves if the OOD detection performance on each task improves (i.e., lower $\delta_k$). Besides, since $(\sum_k \mathbf{1}_{x \in \mathbf{X}_k} e^{\delta_{k}}) (\sum_k 1 - e^{-\delta_k})$ converges to $0$ as $\delta_k$'s converge to $0$ in order of $O(|\sum_k \delta_k|)$, we further know that $H_{TP}$ and $\sum_k H_{OOD,k}$ are equivalent in quantity up to a constant factor.

In Theorem \ref{thm:ce}, we studied how CIL is related to WP and TP. In Theorem \ref{thm:tp_ood}, we showed that TP and OOD bound each other. Now we explicitly give the upper bound of CIL in relation to WP and OOD detection of each task. The detailed proof can be found in Appendix~\ref{prf:cil_with_op_and_ood}. 
\begin{theorem}
\label{thm:cil_with_op_and_ood}
If $H_{OOD, k}(x) \leq \delta_k,\, k = 1, \dots, T$ and $H_{WP}(x) \leq \epsilon$, we have
$$
H_{CIL} (x) \leq \epsilon + (\sum_k \mathbf{1}_{x \in \mathbf{X}_k} e^{\delta_{k}}) (\sum_k 1 - e^{-\delta_k}),$$ 
where $\mathbf{1}_{x \in \mathbf{X}_k}$ is an indicator function.
\end{theorem}

\subsection{Necessary Conditions for Improving CIL}

In Theorem \ref{thm:ce}, we showed that good performances of WP and TP are sufficient to guarantee a good performance of CIL.
In Theorem \ref{thm:cil_with_op_and_ood}, we showed that good performances of WP and OOD are sufficient to guarantee a good performance of CIL. 
For completeness, we study the necessary conditions of a well-performed CIL in this sub-section.

\begin{theorem}
\label{thm:necessary_condition}
If $H_{CIL} (x) \leq \eta$, then there exist
i) a WP, s.t. $H_{WP} (x) \leq \eta$, 
ii) a TP, s.t. $H_{TP} (x) \leq \eta$, and
iii) an OOD detector for each task, s.t. $H_{OOD, k} \leq \eta,\, k = 1, \dots, T$.
\end{theorem}

The detailed proof is given in Appendix~\ref{prf:necessary_condition}. This theorem 
tells that if a good CIL model is trained, then a good WP, a good TP and a good OOD detector for each task {\color{black}are always implied}. 
More importantly, by transforming Theorem \ref{thm:necessary_condition} into its contraposition, we have the following statements:
If for any WP, $H_{WP} (x) > \eta$, then $H_{CIL} (x) > \eta$.
If for any TP, $H_{TP} (x) > \eta$, then $H_{CIL} (x) > \eta$.
If for any OOD detector, $H_{OOD, k} (x) > \eta,\, k=1,\dots,T$, then $H_{CIL} (x) > \eta$.
{Regardless of whether WP and TP (or OOD detection) are defined explicitly or implicitly by a CIL algorithm,
the existence of a good WP and the existence of a good TP or 
OOD detection  are necessary conditions for a good CIL performance.}

\begin{remark}
{\color{black}It is important to note again that our study in this section is based on a CIL model that has already been built. In other words, our study tells the CIL designers what should be achieved in the final model. Clearly, one would also like to know how to design a strong CIL model based on the theoretical results, which also considers catastrophic forgetting (CF). One effective method is to make use of a strong existing TIL algorithm, which can already achieve no or little forgetting (CF), and combine it with a strong OOD detection algorithm (as mentioned earlier, most OOD detection methods can also perform WP). Thus, any improved method from the OOD detection community can be applied to CIL to produce improved CIL systems (see Sections~\ref{sec.betterOOD} and \ref{sec.HAT+CSI}).

Recall in Section~\ref{sec.related}, we reviewed prior works that have tried to use a TIL method for CIL with a task-id prediction method~\cite{von2019continual,Aljundi2016expert,rajasegaran2020itaml,abati2020conditional,henning2021posterior}. However, since they did not know that the key to the success of this approach is a strong OOD detection algorithm, they are quite weak (see Section~\ref{sec:experiments}).}

\end{remark}

\section{New CIL Techniques and Experiments}\label{sec:experiments}
Based on Theorem~\ref{thm:cil_with_op_and_ood}, we have designed several new CIL methods, each of which integrates an existing CL algorithm and an OOD detection algorithm. The OOD detection algorithm that we use can perform both within-task IND prediction (WP) and OOD detection. Our experiments have two goals: (1) to show that a good OOD detection method can help improve the accuracy of an existing CIL algorithm, and (2) to fully compare two of these methods (see some others in Sec.~\ref{sec.collas_pretrain}) with strong baselines to show that they outperform the existing strong baselines considerably.

\subsection{Datasets, CL Baselines and OOD Detection Methods}
\textbf{Datasets and CIL Tasks.} Four popular benchmark image classification datasets are used, from which six CIL problems are created following recent papers \cite{Liu2020,NEURIPS2020_b704ea2c_derpp,Zhu_2021_CVPR_pass}. \textbf{(1) \textit{MNIST}} consists of handwritten images of 10 digits with 60,000/10,000 training/testing samples. We create a CIL problem (\textbf{M-5T}) of 5 tasks with 2 consecutive classes/digits as a task. \textbf{(2) \textit{CIFAR-10}} consists of 32x32 color images of 10 classes with 50,000/10,000 training/testing samples. We create a CIL problem (\textbf{C10-5T}) of 5 tasks with 2 consecutive classes as a task. \textbf{(3) \textit{CIFAR-100}} consists of 60,000 32x32 color images with 50,000/10,000 training/testing samples. We create two CIL problems by splitting 100 classes into 10 tasks (\textbf{C100-10T}) and 20 tasks (\textbf{C100-20T}), where each task has 10 and 5 classes, respectively. \textbf{(4) \textit{Tiny-ImageNet}} has 120,000 64x64 color images of 200 classes with 500/50 images per class for training/testing. We create two CIL problems by splitting 200 classes into 5 tasks (\textbf{T-5T}) and 10 tasks (\textbf{T-10T}), where each task has 40 and 20 classes, respectively.

\textbf{Baseline CL Methods.} We include different families of CL methods: \textit{regularization}, \textit{replay}, \textit{orthogonal projection}, and \textit{parameter isolation}. 
MUC~\cite{Liu2020} and PASS~\cite{Zhu_2021_CVPR_pass} are regularization-based methods.
For replay methods, we use LwF~\cite{Li2016LwF}, iCaRL~\cite{Rebuffi2017}, Mnemonics~\cite{Liu_2020_CVPR}, BiC~\cite{wu2019large}, DER++~\cite{NEURIPS2020_b704ea2c_derpp}, and Co$^2$L~\cite{Cha_2021_ICCV_co2l}. For orthogonal projection, we use OWM~\cite{zeng2019continuous}.
Finally, for parameter isolation, we use CCG~\cite{abati2020conditional}, HyperNet~\cite{von2019continual}, HAT~\cite{Serra2018overcoming}, SupSup~\cite{supsup2020} (Sup), and PR~\cite{henning2021posterior}.\footnote{{iTAML}~\cite{rajasegaran2020itaml} is not included as it requires a batch of test data from the same task to predict the task-id. When each batch has only one test sample, which is our setting, it is very weak. For example, its CIL accuracy is only 33.5\% on C100-10T. {Expert Gate} (EG)~\cite{Aljundi2016expert} is also very weak. Its CIL accuracy is only 43.2 on M-5T. % Both iTAML and EG 
They are much weaker than many baselines. {DER}~\cite{yan2021dynamically} is not included as it expands the network after each task, which is somewhat unfair to other systems as all others do not expand the network. % Due to the expansion, 
DER can generate a large number of parameters after the last task, e.g., 117.6 millions (M) for C100-20T while our proposed methods require 44.6M (HAT+CSI) and 11.6M (Sup+CSI) (refer to Appendix~\ref{apx:n_params}). The average accuracy of DER over the 6 CL experiments is 61.4 while our methods achieve 67.9 (HAT+CSI+c) and 64.9 (Sup+CSI+c) (refer to Tab.~\ref{Tab:maintable}).} We use the official codes for the baselines except for $\text{Co}^2\text{L}$, CCG, and PR. For these three systems, we copy the results from their papers as the code for CCG is not released and we are unable to run $\text{Co}^2\text{L}$ and PR on our machines.

\textbf{OOD Detection Methods}. Two OOD detection methods are used. We combine them with the above existing CL algorithms. Both these methods can also perform \textit{within-task IND prediction} (WP). 

\textbf{(1). ODIN}: Researchers have proposed several methods to improve the OOD detection performance of a trained network by post-processing~\cite{liang2018enhancing,liu2020energy,lee2018simple_md}. ODIN~\cite{liang2018enhancing} is a representative method. It adds perturbation to input and applies a temperature scaling to the softmax output of a trained network.

\textbf{(2). CSI}: It is a recently proposed OOD detection technique~\cite{tack2020csi} that is highly effective. It is based on data and class label augmentation and supervised contrastive learning~\cite{khosla2020supervised}. Its rotation data augmentations create distributional shifted samples to act as negative data for the original samples for contrastive learning. The details of CSI is given in Appendix~\ref{apx:csi}.

\subsection{Training Details and Evaluation Metrics}
\label{sec.training}

\textbf{Training Details.} \label{sec:training_details}
For the backbone structure, we follow \cite{supsup2020,Zhu_2021_CVPR_pass,NEURIPS2020_b704ea2c_derpp}. AlexNet-like architecture~\cite{NIPS2012_c399862d_alexnet} is used for MNIST and ResNet-18~\cite{he2016deep} is used for CIFAR-10. For CIFAR-100 and Tiny-ImageNet, ResNet-18 is also used as CIFAR-10, but the number of channels are doubled to fit more classes. All the methods use the same backbone architecture except for OWM and HyperNet, for which we use their original architectures. OWM uses AlexNet. It is not obvious how to apply the technique to the ResNet structure. HyperNet uses a fully-connected network and ResNet-32 for MNIST and other datasets, respectively. We are unable to change the structure due to model initialization arguments unexplained in the original paper.
For the replay methods, we use memory buffer 200 for MNIST and CIFAR-10 and 2000 for CIFAR-100 and Tiny-ImageNet as in \cite{Rebuffi2017,NEURIPS2020_b704ea2c_derpp}. We use the hyper-parameters suggested by the authors. If we could not reproduce any result, we use 10\% of the training data as a validation set to grid-search for good hyper-parameters. For our proposed methods, we report the hyper-parameters in Appendix~\ref{apx:hyper_params}.
All the results are averages over 5 runs with random seeds.

\textbf{Evaluation Metrics.} 

\textbf{(1).} \textit{Average classification accuracy} over all classes after learning the last task. The final class prediction depends  \textit{prediction methods} (see below). We also report \textit{forgetting rate} % (i.e., backward transfer) 
in Appendix~\ref{apx:forgetting}.

\textbf{(2).} \textit{Average AUC} (Area Under the ROC Curve) over all task models for the evaluation of OOD detection. AUC is the main measure used in OOD detection papers. Using this measure, we show that a better OOD detection method will result in a better CIL performance. Let $\textit{AUC}_{k}$ be the AUC score of task $k$. It is computed by using only the model (or classes) of task $k$ to score the test data of task $k$ as the in-distribution (IND) data and the test data from other tasks as the out-of-distribution (OOD) data. The average AUC score is: $AUC = \sum_{k} \textit{AUC}_{k}/n$, where $n$ is the number of tasks.

It is not straightforward to change existing CL algorithms to include a new OOD detection method that needs training, e.g., CSI, except for TIL (task incremental learning) methods, e.g., HAT and Sup. For HAT and Sup, we can simply switch their methods for learning each task with CSI (see Sec.\ref{sec.HAT+CSI}).

\textbf{Prediction Methods.} The theoretical result in Sec.~\ref{sec.theorem} states that we use Eq.~\ref{eq:cil_in_til_and_tp} to perform the final prediction. The first probability (WP) in Eq.~\ref{eq:cil_in_til_and_tp} is easy to get as we can simply use the softmax values of the classes in each task. However, the second probability (TP) in Eq.~\ref{eq:cil_in_til_and_tp} is tricky as each task is learned without the data of other tasks. There can be many options.
We take the following approaches for prediction (which are a special case of Eq.~\ref{eq:cil_in_til_and_tp}, see below):

\textbf{(1).} For those approaches that use a single classification head to include all classes learned so far, we predict as follows (which is also the approach taken by the existing papers.) 
\begin{align}
    \hat{y} = \argmax f(x)
\end{align}
where $f(x)$ is the logit output of the network. 

\textbf{(2).} For multi-head methods (e.g., HAT, HyperNet, and Sup), which use one head for each task, we use the concatenated output as
\begin{align}
    \hat{y} = \argmax \bigoplus_{k} f(x)_{k} \label{eq:cil_pred}
\end{align}
where $\bigoplus$ indicate concatenation and $f(x)_k$ is the output of task $k$.\footnote{The {Sup} paper proposed an one-shot task-id prediction assuming that the test instances come in a batch and all belong to the same task like iTAML. We assume a single test instance per batch. Its task-id prediction results in accuracy of 50.2 on C10-5T, which is much lower than 62.6 by using Eq.~\ref{eq:cil_pred}. The task-id prediction of {HyperNet} also works poorly. The accuracy by its id prediction is 49.34 on C10-5T while it is 53.4 using Eq.~\ref{eq:cil_pred}. {PR} uses entropy to find task-id. Among many variations of PR, we use the variations that perform the best for each dataset with exemplar-free and single sample per batch at testing (i.e., no PR-BW).}

These methods (in fact, they are the same method used in two different settings) is a special case of Eq.~\ref{eq:cil_in_til_and_tp} if we define $OOD_k$ as $\sigma(\max f(x)_k )$, where $\sigma$ is the sigmoid. Hence, the theoretical results in Sec.~\ref{sec.theorem} are still applicable. We present a detailed explanation about this prediction method and some other options in Appendix~\ref{apx:diff_tp}. These two approaches work quite well.

\subsection{Better OOD Detection Produces Better CIL Performance}\label{sec.betterOOD}
The key theoretical result in Sec.~\ref{sec.theorem} is that better OOD detection will produce better CIL performance. Recall our considered methods ODIN and CSI can perform both WP and OOD detection.

\textbf{Applying ODIN.} 
We first train the baseline models using their original algorithms, and then apply temperature scaling and input noise of ODIN at testing for each task (no training data needed).
More precisely, the output of class $j$ in task $k$ changes by temperature scaling factor $\tau_{k}$ of task $k$ as
\begin{align}
    s(x; \tau_k)_j = e^{f(x)_{kj} / \tau_k } / \sum_{j} e^{f(x)_{kj} / \tau_{k}} \label{eq:odin_softmax}
\end{align}
and the input changes by the noise factor $\epsilon_k$ as
\begin{align}
    \tilde{x} = x - \epsilon_k \text{sign} (-\nabla_x \log s (x; \tau_{k})_{\hat{y}} ) \label{eq:odin_perturbation}
\end{align}
where $\hat{y}$ is the class with the maximum output value in task $k$. This is a positive adversarial example inspired by \cite{goodfellow2015explaining}. The values $\tau_k$ and $\epsilon_k$ are hyper-parameters and we use the same values for all tasks except for PASS, for which we had to use a validation set to tune $\tau_k$ (see Appendix~\ref{apx:additional_odin}).

\begin{wraptable}[29]{r}{2.3in}
\vspace{-4.6mm}
\caption{Performance comparison based on C100-10T between the original output and the output post-processed with OOD detection technique ODIN. Note that ODIN is not applicable to iCaRL and Mnemonics as they are not based on softmax but some distance functions. The results for other datasets are reported in Appendix~\ref{apx:additional_odin}.
} 
\begin{tabular}{llcc}
\toprule
Method & OOD & AUC & CIL \\
\midrule
\multirow{2}{*}{OWM} & Original & 71.31 & 28.91 \\
{} & ODIN & 70.06 & 28.88 \\
\hline
\multirow{2}{*}{MUC} & Original & 72.69 & 30.42 \\
{} & ODIN & 72.53 & 29.79 \\
\hline
\multirow{2}{*}{PASS} & Original & 69.89 & 33.00 \\
{} & ODIN & 69.60 & 31.00 \\
\hline
\multirow{2}{*}{LwF} & Original & 88.30 & 45.26 \\
{} & ODIN & 87.11 & 51.82 \\
\hline
\multirow{2}{*}{BiC} & Original & 87.89 & 52.92 \\
{} & ODIN & 86.73 & 48.65 \\
\hline
\multirow{2}{*}{DER++} & Original & 85.99 & 53.71 \\
{} & ODIN & 88.21 & 55.29 \\
\hline
\multirow{2}{*}{HAT} & Original & 77.72 & 41.06 \\
{} & ODIN & 77.80 & 41.21 \\
\hline
\multirow{2}{*}{HyperNet} & Original & 71.82 & 30.23 \\
{} & ODIN & 72.32 & 30.83 \\
\hline
\multirow{2}{*}{Sup} & Original & 79.16 & 44.58 \\
{} & ODIN & 80.58 & 46.74 \\
\bottomrule
\end{tabular}
\label{Tab:odin}
\end{wraptable}
Tab.~\ref{Tab:odin} gives the results for C100-10T. The CIL results clearly show that the CIL performance increases if the AUC increases with ODIN. For instance, the CIL of DER++ and Sup improves from 53.71 to 55.29 and 44.58 to 46.74, respectively, as the AUC increases from 85.99 to 88.21 and 79.16 to 80.58.
It shows that when this method is incorporated into each task model in existing trained CIL network, the CIL performance of the original method improves. We note that ODIN does not always improve the average AUC.
For those experienced a decrease in AUC, the CIL performance also decreases except LwF. The inconsistency of LwF is due to its severe classification bias towards later tasks as discussed in BiC~\cite{wu2019large}. The temperature scaling in ODIN has a similar effect as 
the bias correction in BiC, and the CIL of LwF becomes close to that of BiC after the correction. Regardless of whether ODIN improves AUC or not, the positive correlation between AUC and CIL (except LwF) verifies the efficacy of Theorem~\ref{thm:cil_with_op_and_ood}, indicating better OOD detection results in better CIL performances.

\textbf{Applying CSI.} We now apply the OOD detection method CSI. Due to its sophisticated data augmentation, supervised constrative learning and results ensemble, it is hard to apply CSI to other baselines without fundamentally change them except for HAT and Sup (SupSup) as these methods are parameter isolation-based TIL methods. We can simply replace their model for training each task with CSI wholesale (the full detail is given in Appendix~\ref{apx:csi}). As mentioned earlier, both HAT and SupSup as TIL methods have almost no forgetting.

Tab.~\ref{Tab:odin_csi} reports the results of using CSI and ODIN. ODIN is a weaker OOD method than CSI.
Both HAT and Sup improve greatly as the systems are equipped with a better OOD detection method CSI.
These experiment results empirically demonstrate the efficacy of Theorem~\ref{thm:cil_with_op_and_ood}, i.e., the CIL performance can be improved if a better OOD detection method is used.

\begin{table}[t]
\centering
\caption{
Average CIL and AUC of HAT and Sup with OOD detection methods ODIN and CSI. ODIN is a traditional OOD detection method while CSI is a recent OOD detection method known to be better than ODIN. As CL methods produce better OOD detection performance by CSI, their CIL performances are better than the ODIN counterparts.
}
\begin{tabular}{l l c c c c c c c c c c}
\toprule
\multicolumn{1}{c}{CL} & \multicolumn{1}{c}{OOD} & \multicolumn{2}{c}{C10-5T} & \multicolumn{2}{c}{C100-10T} &  \multicolumn{2}{c}{C100-20T} & \multicolumn{2}{c}{T-5T} & \multicolumn{2}{c}{T-10T} \\
{} & {} & AUC & CIL & AUC & CIL & AUC & CIL & AUC & CIL & AUC & CIL \\
\midrule
\multirow{2}{*}{HAT} & 
ODIN &
82.5 & 62.6 &
77.8 & 41.2 &
75.4 & 25.8 &
72.3 & 38.6 &
71.8 & 30.0 \\
{} & 
CSI &
91.2 & 87.8 & % 
84.5 & 63.3 & % 
86.5 & 54.6 & % 
76.5 & 45.7 &
78.5 & 47.1 \\ % 
\hline
\multirow{2}{*}{Sup} & 
ODIN &
82.4 & 62.6 &
80.6 & 46.7 &
81.6 & 36.4 &
74.0 & 41.1 &
74.6 & 36.5 \Tstrut \\
{} & 
CSI &
91.6 & 86.0 & % 
86.8 & 65.1 & % 
88.3 & 60.2 & % 
77.1 & 48.9 & %
79.4 & 45.7 \\
\bottomrule
\end{tabular}
\label{Tab:odin_csi}
\vspace{-3mm}
\end{table}

\begin{table}[t]
\centering
\caption{
Average accuracy after all tasks are learned. Exemplar-free methods are italicized. $\dagger$ indicates that in their original papers, PASS and Mnemonics are pre-trained with the first half of the classes. Their results with pre-train are 50.1 and 53.5 on C100-10T, respectively, which are still much lower than the proposed HAT+CSI and Sup+CSI without pre-training. We do not use pre-training in our experiment for fairness. $*$ indicates that iCaRL and Mnemonics report average incremental accuracy in their original papers. We report average accuracy over all classes after all tasks are learned.
}
\begin{tabular}{l c c c c c c}
\toprule
\multirow{1}{*}{Method}  &  \multicolumn{1}{c}{M-5T} & \multicolumn{1}{c}{C10-5T}  &  \multicolumn{1}{c}{C100-10T} &  \multicolumn{1}{c}{C100-20T} &  \multicolumn{1}{c}{T-5T} & \multicolumn{1}{c}{T-10T} \\
\midrule
\textit{OWM} & 95.8\scalebox{0.9}{$\pm$0.13} & 51.8\scalebox{0.9}{$\pm$0.05} & 28.9\scalebox{0.9}{$\pm$0.60} & 24.1\scalebox{0.9}{$\pm$0.26} & 10.0\scalebox{0.9}{$\pm$0.55} & 8.6\scalebox{0.9}{$\pm$0.42} \Tstrut \\
\textit{MUC} & 74.9\scalebox{0.9}{$\pm$0.46} & 52.9\scalebox{0.9}{$\pm$1.03} & 30.4\scalebox{0.9}{$\pm$1.18} & 14.2\scalebox{0.9}{$\pm$0.30} & 33.6\scalebox{0.9}{$\pm$0.19} & 17.4\scalebox{0.9}{$\pm$0.17} \\
\textit{PASS}$^{\dagger}$ & 76.6\scalebox{0.9}{$\pm$1.67} & 47.3\scalebox{0.9}{$\pm$0.98} & 33.0\scalebox{0.9}{$\pm$0.58} & 25.0\scalebox{0.9}{$\pm$0.69} & 28.4\scalebox{0.9}{$\pm$0.51} & 19.1\scalebox{0.9}{$\pm$0.46} \\
LwF & 85.5\scalebox{0.9}{$\pm$3.11} & 54.7\scalebox{0.9}{$\pm$1.18} & 45.3\scalebox{0.9}{$\pm$0.75} & 44.3\scalebox{0.9}{$\pm$0.46} & 32.2\scalebox{0.9}{$\pm$0.50} & 24.3\scalebox{0.9}{$\pm$0.26} \\
iCaRL$^*$  & 96.0\scalebox{0.9}{$\pm$0.43} & 63.4\scalebox{0.9}{$\pm$1.11} & 51.4\scalebox{0.9}{$\pm$0.99} & 47.8\scalebox{0.9}{$\pm$0.48} & 37.0\scalebox{0.9}{$\pm$0.41} & 28.3\scalebox{0.9}{$\pm$0.18} \\
Mnemonics$^{\dagger *}$ & 96.3\scalebox{0.9}{$\pm$0.36} & 64.1\scalebox{0.9}{$\pm$1.47} & 51.0\scalebox{0.9}{$\pm$0.34} & 47.6\scalebox{0.9}{$\pm$0.74} & 37.1\scalebox{0.9}{$\pm$0.46} & 28.5\scalebox{0.9}{$\pm$0.72} \\
BiC & 94.1\scalebox{0.9}{$\pm$0.65} & 61.4\scalebox{0.9}{$\pm$1.74} & 52.9\scalebox{0.9}{$\pm$0.64} & 48.9\scalebox{0.9}{$\pm$0.54} & 41.7\scalebox{0.9}{$\pm$0.74} & 33.8\scalebox{0.9}{$\pm$0.40} \\
DER++ & 95.3\scalebox{0.9}{$\pm$0.69} & 66.0\scalebox{0.9}{$\pm$1.20} & 53.7\scalebox{0.9}{$\pm$1.30} & 46.6\scalebox{0.9}{$\pm$1.44} & 35.8\scalebox{0.9}{$\pm$0.77} & 30.5\scalebox{0.9}{$\pm$0.47} \\
Co$^2$L &  & 65.6 &  &  &  &  \\
\textit{CCG}  & 97.3 & 70.1 &  &  &  & \\
\textit{HAT} & 81.9\scalebox{0.9}{$\pm$3.74} & 62.7\scalebox{0.9}{$\pm$1.45} & 41.1\scalebox{0.9}{$\pm$0.93} & 25.6\scalebox{0.9}{$\pm$0.51} & 38.5\scalebox{0.9}{$\pm$1.85} & 29.8\scalebox{0.9}{$\pm$0.65} \\
\textit{HyperNet} & 56.6\scalebox{0.9}{$\pm$4.85} & 53.4\scalebox{0.9}{$\pm$2.19} & 30.2\scalebox{0.9}{$\pm$1.54} & 18.7\scalebox{0.9}{$\pm$1.10} & 7.9\scalebox{0.9}{$\pm$0.69} & 5.3\scalebox{0.9}{$\pm$0.50} \\
\textit{Sup} & 70.1\scalebox{0.9}{$\pm$1.51} & 62.4\scalebox{0.9}{$\pm$1.45} & 44.6\scalebox{0.9}{$\pm$0.44} & 34.7\scalebox{0.9}{$\pm$0.30} & 41.8\scalebox{0.9}{$\pm$1.50} & 36.5\scalebox{0.9}{$\pm$0.36} \\
\textit{PR-Ent} & 74.1 & 61.9 & 45.2 & & & \\
\hline
\textit{HAT+CSI} & 94.4\scalebox{0.9}{$\pm$0.26} & 87.8\scalebox{0.9}{$\pm$0.71} & 63.3\scalebox{0.9}{$\pm$1.00} & 54.6\scalebox{0.9}{$\pm$0.92} & 45.7\scalebox{0.9}{$\pm$0.26} & 47.1\scalebox{0.9}{$\pm$0.18} \\
\textit{Sup+CSI} & 80.7\scalebox{0.9}{$\pm$2.71} & 86.0\scalebox{0.9}{$\pm$0.41} & 65.1\scalebox{0.9}{$\pm$0.39} & 60.2\scalebox{0.9}{$\pm$0.51} & 48.9\scalebox{0.9}{$\pm$0.25} & 45.7\scalebox{0.9}{$\pm$0.76} \\
HAT+CSI+c & 96.9\scalebox{0.9}{$\pm$0.30} & 88.0\scalebox{0.9}{$\pm$0.48} & 65.2\scalebox{0.9}{$\pm$0.71} & 58.0\scalebox{0.9}{$\pm$0.45} & 51.7\scalebox{0.9}{$\pm$0.37} & 47.6\scalebox{0.9}{$\pm$0.32} \\
Sup+CSI+c & 81.0\scalebox{0.9}{$\pm$2.30} & 87.3\scalebox{0.9}{$\pm$0.37} & 65.2\scalebox{0.9}{$\pm$0.37} & 60.5\scalebox{0.9}{$\pm$0.64} & 49.2\scalebox{0.9}{$\pm$0.28} & 46.2\scalebox{0.9}{$\pm$0.53} \\ 
\bottomrule
\end{tabular}
\label{Tab:maintable}
 \vspace{-3mm}
\end{table}

\subsection{Full Comparison of HAT+CSI and Sup+CSI with Baselines} \label{sec.HAT+CSI}
We now make a full comparison of the two strong systems (HAT+CSI and Sup+CSI) designed based on the theoretical results. These combinations are particularly attractive because both HAT and Sup are TIL systems and have little or no CF. Then a strong OOD method (that can also perform WP (within-task/IND prediction) will result in a strong CIL method.  
Since HAT and Sup are exemplar-free CL methods, HAT+CSI and Sup+CSI also do not need to save any previous task data. Tab.~\ref{Tab:maintable} shows that HAT and Sup equipped with CSI outperform the baselines by large margins.
DER++, the best replay method, achieves 66.0 and 53.7 on C10-5T and C100-10T, respectively, while HAT+CSI achieves 87.8 and 63.3 {and Sup+CSI achieves 86.0 and 65.1}. The large performance gap remains consistent in more challenging problems, T-5T and T-10T. We note that Sup works very poorly on M-5T, but Sup+CSI improved it drastically, although still very weak compared to HAT+CSI.

Due to the definition of OOD in the prediction method and the fact that each task is trained separately in HAT and Sup, the outputs $f(x)_k$ from different tasks can be in different scales, which will result in incorrect predictions. To deal with the problem, we can calibrate the output as $\alpha_k f(x)_k + \beta_k$ and use $OOD_k = \sigma ( \alpha_k f(x)_k + \beta_k )$. The optimal $\alpha_k^*$ and $\beta_k^*$ for each task $k$ can be found by optimization with a memory buffer to save a very small number of training examples from previous tasks like that in the replay-based methods. We refer the calibrated methods as HAT+CSI+c and Sup+CSI+c. They are trained by using the memory buffer of the same size as the replay methods (see Sec.~\ref{sec.training}). Tab.~\ref{Tab:maintable} shows that the calibration improves from their memory free versions, i.e., without calibration. We provide the details about how to train the calibration parameters $\alpha_k$ and $\beta_k$ in Appendix~\ref{apx:calibration}.

{\color{black}As shown in Theorem~\ref{thm:ce}, the CIL performance also depends on the TIL (WP) performance. We compare the TIL accuracies of the baselines and our methods in Tab.~\ref{Tab:til}. Our systems again outperform the baselines by large margins on more challenging datasets (e.g., CIFAR100 and Tiny-ImageNet).
}

\begin{table}[t]
\centering
\caption{
{TIL (WP) results of 3 best performing baselines and our methods. The full results are given in Appendix~\ref{apx:til_results}. The calibrated versions (+c) of our methods are omitted as calibration does not affect TIL performances.}
}
\begin{tabular}{l c c c c c c}
\toprule
\multirow{1}{*}{Method}  &  \multicolumn{1}{c}{M-5T} & \multicolumn{1}{c}{C10-5T}  &  \multicolumn{1}{c}{C100-10T} &  \multicolumn{1}{c}{C100-20T} &  \multicolumn{1}{c}{T-5T} & \multicolumn{1}{c}{T-10T} \\
\midrule
DER++ & 99.7\scalebox{0.9}{$\pm$0.08} & 92.0\scalebox{0.9}{$\pm$0.54} & 84.0\scalebox{0.9}{$\pm$9.43} & 86.6\scalebox{0.9}{$\pm$9.44} & 57.4\scalebox{0.9}{$\pm$1.31} & 60.0\scalebox{0.9}{$\pm$0.74} \\
HAT & 99.9\scalebox{0.9}{$\pm$0.02} & 96.7\scalebox{0.9}{$\pm$0.18} & 84.0\scalebox{0.9}{$\pm$0.23} & 85.0\scalebox{0.9}{$\pm$0.98} & 61.2\scalebox{0.9}{$\pm$0.72} & 63.8\scalebox{0.9}{$\pm$0.41} \\
Sup & 99.6\scalebox{0.9}{$\pm$0.01} & 96.6\scalebox{0.9}{$\pm$0.21} & 87.9\scalebox{0.9}{$\pm$0.27} & 91.6\scalebox{0.9}{$\pm$0.15} & 64.3\scalebox{0.9}{$\pm$0.24} & 68.4\scalebox{0.9}{$\pm$0.22} \\
\hline
HAT+CSI & 99.9\scalebox{0.9}{$\pm$0.00} & 98.7\scalebox{0.9}{$\pm$0.06} & 92.0\scalebox{0.9}{$\pm$0.37} & 94.3\scalebox{0.9}{$\pm$0.06} & 68.4\scalebox{0.9}{$\pm$0.16} & 72.4\scalebox{0.9}{$\pm$0.21} \\
Sup+CSI & 99.0\scalebox{0.9}{$\pm$0.08} & 98.7\scalebox{0.9}{$\pm$0.07} & 93.0\scalebox{0.9}{$\pm$0.13} & 95.3\scalebox{0.9}{$\pm$0.20} & 65.9\scalebox{0.9}{$\pm$0.25} & 74.1\scalebox{0.9}{$\pm$0.28} \\
\bottomrule
\end{tabular}
\label{Tab:til}
\vspace{-2mm}
\end{table}
{\color{black}\subsection{Implications for Existing CL Methods, Open-World Learning and Future Research} \label{sec.collas_pretrain}

{\color{black}
\textbf{Implication for regularization and replay methods.} Regularization-based (exemplar-free) methods try to protect important parameters of old tasks to mitigate CF. However, since the training of each task does not consider OOD detection, TP will be weak, which causes difficulty for \textit{inter-task class separation} (ICS) and thus low CL accuracy. Replay-based methods are better as the replay data from old tasks can be naturally regarded as OOD data for the current task, then a better OOD model is built, which improves TP. 
However, since the replay data is small. the OOD model is sub-optimal, especially for earlier tasks as their training cannot see any future task data. Thus for both approaches, it will be beneficial to consider CF and OOD together in learning each task (e.g., \cite{kim2022multi}). 

}

\textbf{Implication for open-world learning.} {\color{black}Since our theory says that CL needs OOD detection, and OOD detection is also the first step in open-world learning (OWL), CL and OWL naturally work together to achieve \textit{self-motivated open-world continual learning}~\cite{liu2021self} for autonomous learning or AI autonomy. That is, the AI agent can continually discover new tasks (OOD detection) and incrementally learn the tasks (CL) all on its own with no involvement of human engineers. {\color{black}Further afield, this work is also related to curiosity-driven self-supervised learning~\cite{pathak2017curiosity} in reinforcement learning and 3D navigation.}
}

\textbf{Limitation and future work.} {\color{black}The proposed theory provides a principled guidance on what needs to be done in order to achieve good CIL results, but it gives no guidance on how to do it. Although two example techniques are presented and evaluated, they are empirical. There are many options to define WP and TP (or OOD). An idea in \citep{guo2022online} may be helpful in this regard. \citep{guo2022online} argues that a continual learner should learn \textit{holistic feature representations} of the input data, meaning to learn as many features as possible from the input data. The rationale is that 
if the system can learn all possible features from each task, then a future task does not have to learn those shared/intersecting features by modifying the parameters, which will result in less CF and also better ICS. {A full representation of the IND data also improves OOD detection because the OOD score of a data point is basically the distance between the data point and the IND distribution. Only capturing a subset of features (e.g., by cross entropy) will result in poor OOD detection~\cite{hu2020hrn} because those missing features may be necessary to separate IND and some OOD data.} In our future work, we will study how to optimize WP and TP/OOD and find the necessary conditions for them to do well.} 
}

\section{Conclusion}
{\color{black}This paper proposed a theoretical study on how to solve the highly challenging continual learning (CL) problem.  \textit{class incremental learning} (CIL) (the other popular CL setting is \textit{task incremental learning} (TIL)). The theoretical result provides a principled guidance for designing better CIL algorithms.
The paper first decomposed the CIL prediction into \textit{within-task prediction} (WP) and \textit{task-id prediction} (TP). WP is basically TIL.  The paper further theoretically demonstrated that TP is correlated with \textit{out-of-distribution} (OOD) detection.
It then proved that a good performance of the two is both  necessary and sufficient for good CIL performances. 
Based on the theoretical result, several new CIL methods were designed. They outperform strong baselines in CIL and also in TIL by a large margin. Finally, we also discussed the implications for existing CL techniques and open-world learning.}

\vspace{-1mm}
\section*{Acknowledgments}
\vspace{-1mm}
{\color{black}The work of Gyuhak Kim, Zixuan Ke and Bing Liu was supported in part by a research contract from KDDI, two NSF grants (IIS-1910424 and IIS-1838770), and a DARPA contract HR001120C0023.}

\bibliographystyle{unsrt}
\bibliography{neurips_2022.bib}

\begin{thebibliography}{10}

\bibitem{chen2018lifelong}
Zhiyuan Chen and Bing Liu.
\newblock Lifelong machine learning.
\newblock {\em Synthesis Lectures on Artificial Intelligence and Machine
  Learning}, 12(3):1--207, 2018.

\bibitem{van2019three}
Gido~M van~de Ven and Andreas~S Tolias.
\newblock Three scenarios for continual learning.
\newblock {\em arXiv preprint arXiv:1904.07734}, 2019.

\bibitem{Serra2018overcoming}
Joan Serr{\`{a}}, D{\'{i}}dac Sur{\'{i}}s, Marius Miron, and Alexandros
  Karatzoglou.
\newblock {Overcoming catastrophic forgetting with hard attention to the task}.
\newblock In {\em ICML}, 2018.

\bibitem{supsup2020}
Mitchell Wortsman, Vivek Ramanujan, Rosanne Liu, Aniruddha Kembhavi, Mohammad
  Rastegari, Jason Yosinski, and Ali Farhadi.
\newblock Supermasks in superposition.
\newblock In H.~Larochelle, M.~Ranzato, R.~Hadsell, M.~F. Balcan, and H.~Lin,
  editors, {\em NeurIPS}, 2020.

\bibitem{liu2021self}
Bing Liu, Eric Robertson, Scott Grigsby, and Sahisnu Mazumder.
\newblock Self-initiated open world learning for autonomous ai agents.
\newblock {\em Proceedings of AAAI Symposium on Designing Artificial
  Intelligence for Open Worlds}, 2021.

\bibitem{tack2020csi}
Jihoon Tack, Sangwoo Mo, Jongheon Jeong, and Jinwoo Shin.
\newblock Csi: Novelty detection via contrastive learning on distributionally
  shifted instances.
\newblock In {\em NeurIPS}, 2020.

\bibitem{liang2018enhancing}
Shiyu Liang, Yixuan Li, and R.~Srikant.
\newblock Enhancing the reliability of out-of-distribution image detection in
  neural networks.
\newblock In {\em ICLR}, 2018.

\bibitem{esmaeilpour2022zero}
Sepideh Esmaeilpour, Bing Liu, Eric Robertson, and Lei Shu.
\newblock Zero-shot out-of-distribution detection based on the pretrained model
  clip.
\newblock In {\em Proceedings of the AAAI conference on artificial
  intelligence}, 2022.

\bibitem{wang2022omg}
Mengyu Wang, Yijia Shao, Haowei Lin, Wenpeng Hu, and Bing Liu.
\newblock Cmg: A class-mixed generation approach to out-of-distribution
  detection.
\newblock {\em Proceedings of ECML/PKDD-2022}, 2022.

\bibitem{pentina2014pac}
Anastasia Pentina and Christoph Lampert.
\newblock A pac-bayesian bound for lifelong learning.
\newblock In {\em International Conference on Machine Learning}, pages
  991--999. PMLR, 2014.

\bibitem{karakida2022learning}
Ryo Karakida and Shotaro Akaho.
\newblock Learning curves for continual learning in neural networks:
  Self-knowledge transfer and forgetting.
\newblock In {\em International Conference on Learning Representations}, 2022.

\bibitem{kirkpatrick2017overcoming}
James Kirkpatrick, Razvan Pascanu, Neil Rabinowitz, Joel Veness, Guillaume
  Desjardins, Andrei~A Rusu, Kieran Milan, John Quan, Tiago Ramalho, Agnieszka
  Grabska-Barwinska, et~al.
\newblock Overcoming catastrophic forgetting in neural networks.
\newblock {\em Proceedings of the national academy of sciences},
  114(13):3521--3526, 2017.

\bibitem{Li2016LwF}
Zhizhong Li and Derek Hoiem.
\newblock {Learning Without Forgetting}.
\newblock In {\em ECCV}, pages 614--629. Springer, 2016.

\bibitem{Jung2016less}
Heechul Jung, Jeongwoo Ju, Minju Jung, and Junmo Kim.
\newblock {Less-forgetting learning in deep neural networks}.
\newblock {\em arXiv preprint arXiv:1607.00122}, 2016.

\bibitem{Camoriano2017incremental}
Raffaello Camoriano, Giulia Pasquale, Carlo Ciliberto, Lorenzo Natale, Lorenzo
  Rosasco, and Giorgio Metta.
\newblock {Incremental robot learning of new objects with fixed update time}.
\newblock In {\em ICRA}, 2017.

\bibitem{zenke2017continual}
Friedemann Zenke, Ben Poole, and Surya Ganguli.
\newblock {Continual learning through synaptic intelligence}.
\newblock In {\em ICML}, pages 3987--3995, 2017.

\bibitem{ritter2018online}
Hippolyt Ritter, Aleksandar Botev, and David Barber.
\newblock Online structured laplace approximations for overcoming catastrophic
  forgetting.
\newblock In {\em NeurIPS}, 2018.

\bibitem{schwarz2018progress}
Jonathan Schwarz, Jelena Luketina, Wojciech~M Czarnecki, Agnieszka
  Grabska-Barwinska, Yee~Whye Teh, Razvan Pascanu, and Raia Hadsell.
\newblock Progress \& compress: A scalable framework for continual learning.
\newblock {\em arXiv preprint arXiv:1805.06370}, 2018.

\bibitem{xu2018reinforced}
Ju~Xu and Zhanxing Zhu.
\newblock Reinforced continual learning.
\newblock In {\em NeurIPS}, 2018.

\bibitem{castro2018end}
Francisco~M Castro, Manuel~J Mar{\'\i}n-Jim{\'e}nez, Nicol{\'a}s Guil, Cordelia
  Schmid, and Karteek Alahari.
\newblock End-to-end incremental learning.
\newblock In {\em ECCV}, pages 233--248, 2018.

\bibitem{hu2019overcoming}
Wenpeng Hu, Zhou Lin, Bing Liu, Chongyang Tao, Zhengwei Tao, Jinwen Ma, Dongyan
  Zhao, and Rui Yan.
\newblock Overcoming catastrophic forgetting for continual learning via model
  adaptation.
\newblock In {\em ICLR}, 2019.

\bibitem{Dhar2019CVPR}
Prithviraj Dhar, Rajat~Vikram Singh, Kuan{-}Chuan Peng, Ziyan Wu, and Rama
  Chellappa.
\newblock Learning without memorizing.
\newblock In {\em CVPR}, 2019.

\bibitem{lee2019overcoming}
Kibok Lee, Kimin Lee, Jinwoo Shin, and Honglak Lee.
\newblock Overcoming catastrophic forgetting with unlabeled data in the wild.
\newblock In {\em CVPR}, 2019.

\bibitem{ahn2019neurIPS}
Hongjoon Ahn, Sungmin Cha, Donggyu Lee, and Taesup Moon.
\newblock Uncertainty-based continual learning with adaptive regularization.
\newblock In {\em NeurIPS}, 2019.

\bibitem{Liu2020}
Yu~Liu, Sarah Parisot, Gregory Slabaugh, Xu~Jia, Ales Leonardis, and Tinne
  Tuytelaars.
\newblock More classifiers, less forgetting: A generic multi-classifier
  paradigm for incremental learning.
\newblock In {\em ECCV}, pages 699--716. Springer International Publishing,
  2020.

\bibitem{Zhu_2021_CVPR_pass}
Fei Zhu, Xu-Yao Zhang, Chuang Wang, Fei Yin, and Cheng-Lin Liu.
\newblock Prototype augmentation and self-supervision for incremental learning.
\newblock In {\em CVPR}, 2021.

\bibitem{Rusu2016}
Andrei~A Rusu, Neil~C Rabinowitz, Guillaume Desjardins, Hubert Soyer, James
  Kirkpatrick, Koray Kavukcuoglu, Razvan Pascanu, and Raia Hadsell.
\newblock {Progressive neural networks}.
\newblock {\em arXiv preprint arXiv:1606.04671}, 2016.

\bibitem{Lopez2017gradient}
David Lopez-Paz and Marc’Aurelio Ranzato.
\newblock {Gradient Episodic Memory for Continual Learning}.
\newblock In {\em NeurIPS}, pages 6470--6479, 2017.

\bibitem{Rebuffi2017}
Sylvestre-Alvise Rebuffi, Alexander Kolesnikov, and Christoph~H Lampert.
\newblock {iCaRL: Incremental classifier and representation learning}.
\newblock In {\em CVPR}, pages 5533--5542, 2017.

\bibitem{Chaudhry2019ICLR}
Arslan Chaudhry, Marc’Aurelio Ranzato, Marcus Rohrbach, and Mohamed
  Elhoseiny.
\newblock Efficient lifelong learning with a-gem.
\newblock In {\em ICLR}, 2019.

\bibitem{hou2019learning}
Saihui Hou, Xinyu Pan, Chen~Change Loy, Zilei Wang, and Dahua Lin.
\newblock Learning a unified classifier incrementally via rebalancing.
\newblock In {\em CVPR}, pages 831--839, 2019.

\bibitem{wu2019large}
Yue Wu, Yinpeng Chen, Lijuan Wang, Yuancheng Ye, Zicheng Liu, Yandong Guo, and
  Yun Fu.
\newblock Large scale incremental learning.
\newblock In {\em CVPR}, 2019.

\bibitem{rolnick2019neurIPS}
David Rolnick, Arun Ahuja, Jonathan Schwarz, Timothy~P. Lillicrap, and Greg
  Wayne.
\newblock Experience replay for continual learning.
\newblock In {\em NeurIPS}, 2019.

\bibitem{NEURIPS2020_b704ea2c_derpp}
Pietro Buzzega, Matteo Boschini, Angelo Porrello, Davide Abati, and Simone
  Calderara.
\newblock Dark experience for general continual learning: a strong, simple
  baseline.
\newblock In {\em NeurIPS}, 2020.

\bibitem{rajasegaran2020adaptive}
Jathushan Rajasegaran, Munawar Hayat, Salman Khan, Fahad~Shahbaz Khan, Ling
  Shao, and Ming-Hsuan Yang.
\newblock An adaptive random path selection approach for incremental learning,
  2020.

\bibitem{Liu2020AANets}
Yaoyao Liu, Bernt Schiele, and Qianru Sun.
\newblock Adaptive aggregation networks for class-incremental learning.
\newblock In {\em CVPR}, 2021.

\bibitem{Cha_2021_ICCV_co2l}
Hyuntak Cha, Jaeho Lee, and Jinwoo Shin.
\newblock Co2l: Contrastive continual learning.
\newblock In {\em ICCV}, 2021.

\bibitem{yan2021dynamically}
Shipeng Yan, Jiangwei Xie, and Xuming He.
\newblock Der: Dynamically expandable representation for class incremental
  learning.
\newblock In {\em Proceedings of the IEEE/CVF Conference on Computer Vision and
  Pattern Recognition}, pages 3014--3023, 2021.

\bibitem{wang2022memory}
Liyuan Wang, Xingxing Zhang, Kuo Yang, Longhui Yu, Chongxuan Li, Lanqing Hong,
  Shifeng Zhang, Zhenguo Li, Yi~Zhong, and Jun Zhu.
\newblock Memory replay with data compression for continual learning.
\newblock {\em Proceedings of International Conference on Learning
  Representations (ICLR)}, 2022.

\bibitem{guo2022online}
Yiduo Guo, Bing Liu, and Dongyan Zhao.
\newblock Online continual learning through mutual information maximization.
\newblock In {\em International Conference on Machine Learning}, pages
  8109--8126. PMLR, 2022.

\bibitem{Gepperth2016bio}
Alexander Gepperth and Cem Karaoguz.
\newblock {A bio-inspired incremental learning architecture for applied
  perceptual problems}.
\newblock {\em Cognitive Computation}, 8(5):924--934, 2016.

\bibitem{Kamra2017deep}
Nitin Kamra, Umang Gupta, and Yan Liu.
\newblock {Deep Generative Dual Memory Network for Continual Learning}.
\newblock {\em arXiv preprint arXiv:1710.10368}, 2017.

\bibitem{Shin2017continual}
Hanul Shin, Jung~Kwon Lee, Jaehong Kim, and Jiwon Kim.
\newblock {Continual learning with deep generative replay}.
\newblock In {\em NIPS}, pages 2994--3003, 2017.

\bibitem{wu2018memory}
Chenshen Wu, Luis Herranz, Xialei Liu, Joost van~de Weijer, Bogdan Raducanu,
  et~al.
\newblock Memory replay gans: Learning to generate new categories without
  forgetting.
\newblock In {\em NeurIPS}, 2018.

\bibitem{Seff2017continual}
Ari Seff, Alex Beatson, Daniel Suo, and Han Liu.
\newblock {Continual learning in generative adversarial nets}.
\newblock {\em arXiv preprint arXiv:1705.08395}, 2017.

\bibitem{Kemker2018fearnet}
Ronald Kemker and Christopher Kanan.
\newblock {FearNet: Brain-Inspired Model for Incremental Learning}.
\newblock In {\em ICLR}, 2018.

\bibitem{Rostami2019ijcai}
Mohammad Rostami, Soheil Kolouri, and Praveen~K. Pilly.
\newblock Complementary learning for overcoming catastrophic forgetting using
  experience replay.
\newblock In {\em IJCAI}, 2019.

\bibitem{ostapenko2019learning}
Oleksiy Ostapenko, Mihai Puscas, Tassilo Klein, Patrick Jahnichen, and Moin
  Nabi.
\newblock Learning to remember: A synaptic plasticity driven framework for
  continual learning.
\newblock In {\em CVPR}, pages 11321--11329, 2019.

\bibitem{zeng2019continuous}
Guanxiong Zeng, Yang Chen, Bo~Cui, and Shan Yu.
\newblock Continuous learning of context-dependent processing in neural
  networks.
\newblock {\em Nature Machine Intelligence}, 2019.

\bibitem{guo2022adaptive}
Yiduo Guo, Wenpeng Hu, Dongyan Zhao, and Bing Liu.
\newblock Adaptive orthogonal projection for batch and online continual
  learning.
\newblock In {\em Proceedings of AAAI-2022}, 2022.

\bibitem{chaudhry2020continual}
Arslan Chaudhry, Naeemullah Khan, Puneet~K. Dokania, and Philip H.~S. Torr.
\newblock Continual learning in low-rank orthogonal subspaces, 2020.

\bibitem{ke2020continual}
Zixuan Ke, Bing Liu, and Xingchang Huang.
\newblock Continual learning of a mixed sequence of similar and dissimilar
  tasks.
\newblock In {\em NeurIPS}, 2020.

\bibitem{Mallya2017packnet}
Arun Mallya and Svetlana Lazebnik.
\newblock {PackNet: Adding Multiple Tasks to a Single Network by Iterative
  Pruning}.
\newblock {\em arXiv preprint arXiv:1711.05769}, 2017.

\bibitem{NEURIPS2019_3b220b43compact}
Ching-Yi Hung, Cheng-Hao Tu, Cheng-En Wu, Chien-Hung Chen, Yi-Ming Chan, and
  Chu-Song Chen.
\newblock Compacting, picking and growing for unforgetting continual learning.
\newblock In {\em NeurIPS}, volume~32, 2019.

\bibitem{von2019continual}
Johannes von Oswald, Christian Henning, Jo{\~a}o Sacramento, and Benjamin~F
  Grewe.
\newblock Continual learning with hypernetworks.
\newblock {\em ICLR}, 2020.

\bibitem{Yoon2020Scalable}
Jaehong Yoon, Saehoon Kim, Eunho Yang, and Sung~Ju Hwang.
\newblock Scalable and order-robust continual learning with additive parameter
  decomposition.
\newblock In {\em ICLR}, 2020.

\bibitem{singh2020calibrating}
Pravendra Singh, Vinay~Kumar Verma, Pratik Mazumder, Lawrence Carin, and Piyush
  Rai.
\newblock Calibrating cnns for lifelong learning.
\newblock {\em NeurIPS}, 2020.

\bibitem{he2020momentum}
Kaiming He, Haoqi Fan, Yuxin Wu, Saining Xie, and Ross Girshick.
\newblock Momentum contrast for unsupervised visual representation learning.
\newblock In {\em CVPR}, 2020.

\bibitem{chen2020simple}
Ting Chen, Simon Kornblith, Mohammad Norouzi, and Geoffrey Hinton.
\newblock A simple framework for contrastive learning of visual
  representations.
\newblock In {\em ICML}, 2020.

\bibitem{bulusu2020anomalous}
Saikiran Bulusu, Bhavya Kailkhura, Bo~Li, Pramod~K Varshney, and Dawn Song.
\newblock Anomalous instance detection in deep learning: A survey.
\newblock {\em arXiv preprint arXiv:2003.06979}, 2020.

\bibitem{geng2020recent}
Chuanxing Geng, Sheng-jun Huang, and Songcan Chen.
\newblock Recent advances in open set recognition: A survey.
\newblock {\em IEEE transactions on pattern analysis and machine intelligence},
  2020.

\bibitem{rajasegaran2020itaml}
Jathushan Rajasegaran, Salman Khan, Munawar Hayat, Fahad~Shahbaz Khan, and
  Mubarak Shah.
\newblock itaml: An incremental task-agnostic meta-learning approach.
\newblock In {\em CVPR}, 2020.

\bibitem{abati2020conditional}
Davide Abati, Jakub Tomczak, Tijmen Blankevoort, Simone Calderara, Rita
  Cucchiara, and Ehteshami Bejnordi.
\newblock Conditional channel gated networks for task-aware continual learning.
\newblock In {\em CVPR}, pages 3931--3940, 2020.

\bibitem{Aljundi2016expert}
Rahaf Aljundi, Punarjay Chakravarty, and Tinne Tuytelaars.
\newblock Expert gate: Lifelong learning with a network of experts.
\newblock In {\em CVPR}, 2017.

\bibitem{henning2021posterior}
Christian Henning, Maria Cervera, Francesco D'Angelo, Johannes Von~Oswald,
  Regina Traber, Benjamin Ehret, Seijin Kobayashi, Benjamin~F Grewe, and
  Jo{\~a}o Sacramento.
\newblock Posterior meta-replay for continual learning.
\newblock {\em NeurIPS}, 34, 2021.

\bibitem{lee2021continual}
Sebastian Lee, Sebastian Goldt, and Andrew Saxe.
\newblock Continual learning in the teacher-student setup: Impact of task
  similarity.
\newblock In {\em International Conference on Machine Learning}, pages
  6109--6119. PMLR, 2021.

\bibitem{bennani2020generalisation}
Mehdi~Abbana Bennani, Thang Doan, and Masashi Sugiyama.
\newblock Generalisation guarantees for continual learning with orthogonal
  gradient descent.
\newblock {\em Lifelong Learning Workshop at the ICML}, 2020.

\bibitem{bang2021rainbow}
Jihwan Bang, Heesu Kim, YoungJoon Yoo, Jung-Woo Ha, and Jonghyun Choi.
\newblock Rainbow memory: Continual learning with a memory of diverse samples.
\newblock In {\em Proceedings of the IEEE/CVF Conference on Computer Vision and
  Pattern Recognition}, pages 8218--8227, 2021.

\bibitem{Liu_2020_CVPR}
Yaoyao Liu, Yuting Su, An-An Liu, Bernt Schiele, and Qianru Sun.
\newblock Mnemonics training: Multi-class incremental learning without
  forgetting.
\newblock In {\em CVPR}, 2020.

\bibitem{liu2020energy}
Weitang Liu, Xiaoyun Wang, John Owens, and Yixuan Li.
\newblock Energy-based out-of-distribution detection.
\newblock {\em Advances in Neural Information Processing Systems}, 2020.

\bibitem{lee2018simple_md}
Kimin Lee, Kibok Lee, Honglak Lee, and Jinwoo Shin.
\newblock A simple unified framework for detecting out-of-distribution samples
  and adversarial attacks.
\newblock {\em Advances in neural information processing systems}, 31, 2018.

\bibitem{khosla2020supervised}
Prannay Khosla, Piotr Teterwak, Chen Wang, Aaron Sarna, Yonglong Tian, Phillip
  Isola, Aaron Maschinot, Ce~Liu, and Dilip Krishnan.
\newblock Supervised contrastive learning.
\newblock {\em arXiv preprint arXiv:2004.11362}, 2020.

\bibitem{NIPS2012_c399862d_alexnet}
Alex Krizhevsky, Ilya Sutskever, and Geoffrey~E Hinton.
\newblock Imagenet classification with deep convolutional neural networks.
\newblock In {\em NIPS}, 2012.

\bibitem{he2016deep}
Kaiming He, Xiangyu Zhang, Shaoqing Ren, and Jian Sun.
\newblock Deep residual learning for image recognition.
\newblock In {\em CVPR}, 2016.

\bibitem{goodfellow2015explaining}
Ian~J Goodfellow, Jonathon Shlens, and Christian Szegedy.
\newblock Explaining and harnessing adversarial examples.
\newblock {\em ICLR}, 2015.

\bibitem{kim2022multi}
Gyuhak Kim, Zixuan Ke, and Bing Liu.
\newblock A multi-head model for continual learning via out-of-distribution
  replay.
\newblock {\em arXiv preprint arXiv:2208.09734}, 2022.

\bibitem{pathak2017curiosity}
Deepak Pathak, Pulkit Agrawal, Alexei~A Efros, and Trevor Darrell.
\newblock Curiosity-driven exploration by self-supervised prediction.
\newblock In {\em International conference on machine learning}, pages
  2778--2787. PMLR, 2017.

\bibitem{hu2020hrn}
Wenpeng Hu, Mengyu Wang, Qi~Qin, Jinwen Ma, and Bing Liu.
\newblock Hrn: A holistic approach to one class learning.
\newblock {\em Advances in Neural Information Processing Systems},
  33:19111--19124, 2020.

\bibitem{ramanujan2020s}
Vivek Ramanujan, Mitchell Wortsman, Aniruddha Kembhavi, Ali Farhadi, and
  Mohammad Rastegari.
\newblock What's hidden in a randomly weighted neural network?
\newblock In {\em Proceedings of the IEEE/CVF Conference on Computer Vision and
  Pattern Recognition}, pages 11893--11902, 2020.

\bibitem{inception}
Christian Szegedy, Wei Liu, Yangqing Jia, Pierre Sermanet, Scott Reed, Dragomir
  Anguelov, Dumitru Erhan, Vincent Vanhoucke, and Andrew Rabinovich.
\newblock Going deeper with convolutions.
\newblock In {\em CVPR}, 2015.

\bibitem{hendrycks2019using}
Dan Hendrycks, Mantas Mazeika, Saurav Kadavath, and Dawn Song.
\newblock Using self-supervised learning can improve model robustness and
  uncertainty.
\newblock In {\em NeurIPS}, pages 15663--15674, 2019.

\bibitem{you2017large}
Yang You, Igor Gitman, and Boris Ginsburg.
\newblock Large batch training of convolutional networks.
\newblock {\em arXiv preprint arXiv:1708.03888}, 2017.

\bibitem{loshchilov2016sgdr}
Ilya Loshchilov and Frank Hutter.
\newblock Sgdr: Stochastic gradient descent with warm restarts.
\newblock {\em arXiv preprint arXiv:1608.03983}, 2016.

\bibitem{kim2022continual}
Gyuhak Kim, Sepideh Esmaeilpour, Changnan Xiao, and Bing Liu.
\newblock Continual learning based on ood detection and task masking.
\newblock In {\em CVPR 2022 Workshop on Continual Learning}, 2022.

\end{thebibliography}

\newpage

%%%%%%%%%%%%%%%%%%%%%%%%%%%%%%%%%%%%%%%%%%%%%%%%%%%%%%%%%%%%

%%%%%%%%%%%%%%%%%%%%%%%%%%%%%%%%%%%%%%%%%%%%%%%%%%%%%%%%%%%%

\appendix

\section{Proof of Theorems and Corollaries}
\subsection{Proof of Theorem~\ref{thm:ce}}
\label{prf:ce}
\begin{proof}
Since 
$$
\begin{aligned}
    H_{CIL}(x) &= H(y, \{\mathbf{P}(x \in \mathbf{X}_{k, j} | D)\}_{k, j}) \\
    &= - \sum_{k, j} y_{k, j} \log \mathbf{P}(x \in \mathbf{X}_{k, j} | D) \\
    &= - \log \mathbf{P}(x \in \mathbf{X}_{k_0, j_0} | D),
\end{aligned}
$$
$$
\begin{aligned}
    H_{WP}(x) &= H(\Tilde{y}, \{\mathbf{P}(x \in \mathbf{X}_{k_0, j} | x \in \mathbf{X}_{k_0}, D)\}_{j}) \\
    &= - \sum_{j} y_{k_0, j} \log \mathbf{P}(x \in \mathbf{X}_{k_0, j} | x \in \mathbf{X}_{k_0}, D) \\
    &= - \log \mathbf{P}(x \in \mathbf{X}_{k_0, j_0} | x \in \mathbf{X}_{k_0}, D),
\end{aligned}
$$
and 
$$
\begin{aligned}
    H_{TP}(x) &= H(\Bar{y}, \{\mathbf{P}(x \in \mathbf{X}_k | D)\}_{k}) \\
    &= - \sum_{k} \Bar{y}_{k} \log \mathbf{P}(x \in \mathbf{X}_{k} | D) \\
    &= - \log \mathbf{P}(x \in \mathbf{X}_{k_0} | D),
\end{aligned}
$$
we have 
$$
\begin{aligned}
    H_{CIL}(x) &= - \log \mathbf{P}(x \in \mathbf{X}_{k_0, j_0} | D) \\ 
    &= - \log \mathbf{P}(x \in \mathbf{X}_{k_0, j_0} | x \in \mathbf{X}_{k_0}, D) - \log \mathbf{P}(x \in \mathbf{X}_{k_0} | D) \\
    &= H_{WP}(x) + H_{TP}(x) \\
    &\leq \epsilon + \delta. 
\end{aligned}
$$
\end{proof}

\subsection{Proof of Corollary~\ref{cor:expectation}.}
\label{prf:expectation}
\begin{proof}
By proof of Theorem~\ref{thm:ce}, we have 
$$
H_{CIL} (x) = H_{WP} (x) + H_{TP} (x).
$$
Taking expectations on both sides, we have i)
$$
\begin{aligned}
    \mathbb{E}_{x\sim U(\mathbf{X})} [H_{CIL} (x)] &= \mathbb{E}_{x \sim U(\mathbf{X})} [H_{WP} (x)] +\mathbb{E}_{x\sim U(\mathbf{X})} [H_{TP} (x)] \\
    &\leq \mathbb{E}_{x \sim U(\mathbf{X})} [H_{WP} (x)] + \delta.
\end{aligned}
$$
and ii)
$$
\begin{aligned}
    \mathbb{E}_{x\sim U(\mathbf{X})} [H_{CIL} (x)] &= \mathbb{E}_{x \sim U(\mathbf{X})} [H_{WP} (x)] + \mathbb{E}_{x\sim U(\mathbf{X})} [H_{TP} (x)] \\
    &\leq \epsilon + \mathbb{E}_{x\sim U(\mathbf{X})} [H_{TP} (x)].
\end{aligned}
$$
\end{proof}

\subsection{Proof of Theorem~\ref{thm:tp_ood}.}
\label{prf:tp_ood}
\begin{proof}
i) Assume $x \in \mathbf{X}_{k_0}$. 

For $k = k_0$, we have
$$
\begin{aligned}
    H_{OOD, k_0} (x) &= - \log \mathbf{P}'_{k_0} (x \in \mathbf{X}_{k_0} | D) \\
    &= - \log \mathbf{P}(x \in \mathbf{X}_{k_0} | D) \\
    &= H_{TP} (x) \leq \delta.
\end{aligned}
$$
For $k \neq k_0$, we have
$$
\begin{aligned}
    H_{OOD, k} (x) &= - \log \mathbf{P}'_k (x \notin \mathbf{X}_k | D) \\
    &= - \log (1 - \mathbf{P}'_k (x \in \mathbf{X}_k | D))\\
    &= - \log (1 - \mathbf{P} (x \in \mathbf{X}_k | D))\\
    &= - \log \mathbf{P}(x \in \cup_{k' \neq k} \mathbf{X}_{k'} | D) \\
    &\leq - \log \mathbf{P}(x \in \mathbf{X}_{k_0} | D) \\
    &= H_{TP} (x) \leq \delta.
\end{aligned}
$$

ii) Assume $x \in \mathbf{X}_{k_0}$. 

For $k = k_0$, by $H_{OOD, {k_0}} (x) \leq \delta_{k_0}$, we have 
$$- \log \mathbf{P}'_{k_0} (x \in \mathbf{X}_{k_0} | D) \leq \delta_{k_0},$$ 
which means 
$$\mathbf{P}'_{k_0} (x \in \mathbf{X}_{k_0} | D) \geq e^{-\delta_{k_0}}.$$
For $k \neq k_0$, by $H_{OOD, {k}} (x) \leq \delta_{k}$, we have 
$$- \log \mathbf{P}'_{k} (x \notin \mathbf{X}_{k} | D) \leq \delta_{k},$$
which means 
$$\mathbf{P}'_{k} (x \in \mathbf{X}_{k} | D) \leq 1 - e^{-\delta_{k}}.$$

Therefore, we have
$$
\begin{aligned}
    \mathbf{P} (x \in \mathbf{X}_{k_0} | D) &= \frac{\mathbf{P}'_{k_0} (x \in \mathbf{X}_{k_0} |D)}{\sum_{k'} \mathbf{P}'_{k'} (x \in \mathbf{X}_{k'} |D)} \\
    &\geq \frac{e^{-\delta_{k_0}}}{1 + \sum_{k\neq k_0} 1 - e^{-\delta_k}} \\
    &= \frac{e^{-\delta_{k_0}}}{e^{-\delta_{k_0}} + \sum_{k} 1 - e^{-\delta_k}} \\
    &= \frac{1}{1 + e^{\delta_{k_0}}\sum_{k} 1 - e^{-\delta_k}}.
\end{aligned}
$$

Hence, 
$$
\begin{aligned}
    H_{TP} (x) &= -\log \mathbf{P} (x \in \mathbf{X}_{k_0} | D) \\
    &\leq - \log \frac{1}{1 + e^{\delta_{k_0}}\sum_{k} 1 - e^{-\delta_k}} \\
    &= \log [1 + e^{\delta_{k_0}}\sum_{k} 1 - e^{-\delta_k}] \\
    &\leq e^{\delta_{k_0}} (\sum_k 1 - e^{-\delta_k}) \\
    &= (\sum_k \mathbf{1}_{x \in \mathbf{X}_k} e^{\delta_{k}}) (\sum_k 1 - e^{-\delta_k}).
\end{aligned}
$$

\end{proof}

\subsection{Proof of Theorem~\ref{thm:cil_with_op_and_ood}.}
\label{prf:cil_with_op_and_ood}
\begin{proof}
Using Theorem~\ref{thm:ce} and \ref{thm:tp_ood},
$$
\begin{aligned}
    H_{CIL}(x) &= - \log \mathbf{P}(x \in \mathbf{X}_{k_0, j_0} | D) \\ 
    &= - \log \mathbf{P}(x \in \mathbf{X}_{k_0, j_0} | x \in \mathbf{X}_{k_0}, D) - \log \mathbf{P}(x \in \mathbf{X}_{k_0} | D) \\
    &= H_{WP}(x) + H_{TP}(x) \\
    &\leq \epsilon + H_{TP}(x) \\
    &\leq \epsilon + (\sum_k \mathbf{1}_{x \in \mathbf{X}_k}  e^{\delta_{k}}) (\sum_k 1 - e^{-\delta_k})
\end{aligned}
$$
\end{proof}

\subsection{Proof of Theorem~\ref{thm:necessary_condition}.}
\label{prf:necessary_condition}
\begin{proof}
i) Assume $x \in \mathbf{X}_{k_0, j_0} \subset \mathbf{X}_{k_0}$. 

Define $\mathbf{P} (x \in \mathbf{X}_{k, j} | x \in \mathbf{X}_k, D) = \mathbf{P} (x \in \mathbf{X}_{k, j} | D)$.

According to proof of Theorem~\ref{thm:ce}, 
$$
\begin{aligned}
H_{WP} (x) &= -\log \mathbf{P} (x \in \mathbf{X}_{k_0, j_0} | x \in \mathbf{X}_{k_0}, D), \\
H_{CIL} (x) &= -\log \mathbf{P} (x \in \mathbf{X}_{k_0, j_0} | D).
\end{aligned}
$$
Hence, we have
$$
\begin{aligned}
H_{WP} (x) &= -\log \mathbf{P} (x \in \mathbf{X}_{k_0, j_0} | x \in \mathbf{X}_{k_0}, D) \\
&= -\log \mathbf{P} (x \in \mathbf{X}_{k_0, j_0} | D) \\ 
&= H_{CIL} (x) \leq \eta.
\end{aligned}
$$

ii) Assume $x \in \mathbf{X}_{k_0, j_0} \subset \mathbf{X}_{k_0}$. 

Define $\mathbf{P} (x \in \mathbf{X}_{k} | D) = \sum_j \mathbf{P} (x \in \mathbf{X}_{k, j} | D)$.

According to proof of Theorem~\ref{thm:ce}, 
$$
\begin{aligned}
H_{TP} (x) &= -\log \mathbf{P} (x \in \mathbf{X}_{k_0} | D), \\
H_{CIL} (x) &= -\log \mathbf{P} (x \in \mathbf{X}_{k_0, j_0} | D).
\end{aligned}
$$
Hence, we have 
$$
\begin{aligned}
H_{TP} (x) &= -\log \mathbf{P} (x \in \mathbf{X}_{k_0} | D) \\
&= -\log \sum_j \mathbf{P} (x \in \mathbf{X}_{k_0, j} | D) \\
&\leq -\log \mathbf{P} (x \in \mathbf{X}_{k_0, j_0} | D) \\
&= H_{CIL} (x) \leq \eta.
\end{aligned}
$$

iii) Assume $x \in \mathbf{X}_{k_0, j_0} \subset \mathbf{X}_{k_0}$. 

Define $\mathbf{P}'_i (x \in \mathbf{X}_{k} | D) = \mathbf{P} (x \in \mathbf{X}_{k} | D) = \sum_j \mathbf{P} (x \in \mathbf{X}_{k, j} | D)$.

According to proof of Theorem~\ref{thm:necessary_condition} ii), we have
$$
H_{TP} (x) \leq \eta.
$$
According to proof of Theorem~\ref{thm:tp_ood} i), we have 
$$
H_{OOD, i} (x) \leq H_{TP} (x).
$$
Therefore, 
$$
H_{OOD, i} (x) \leq H_{TP} (x) \leq \eta.
$$

\end{proof}

\section{Additional Results and Explanation Regarding Table 1 in the Main Paper} \label{apx:additional_odin}
In Sec.~\ref{sec.betterOOD}, we showed that a better OOD detection improves CIL performance. For the post-processing method ODIN, we only reported the results on C100-10T due to space limitations. Tab.~\ref{Tab:odin_additional} shows the results on the other datasets.
\begin{table}
    \centering
    \caption{
    Performance comparison between the original output and output post-processed with OOD detection technique ODIN. Note that ODIN is not applicable to iCaRL and Mnemonics as they are not based on softmax but some distance functions. The result for C100-10T are reported in the main paper.
    }
    \resizebox{5.5in}{!}{
    \begin{tabular}{ll cc cc cc cc cc}
    \toprule
    {} & {} & \multicolumn{2}{c}{M-5T} & \multicolumn{2}{c}{C10-5T} & \multicolumn{2}{c}{C100-20T} & \multicolumn{2}{c}{T-5T} & \multicolumn{2}{c}{T-10T} \\
    Method & OOD & AUC & CIL & AUC & CIL & AUC & CIL & AUC & CIL & AUC & CIL \\
    \midrule
    \multirow{2}{*}{OWM} & Original & 99.13 & 95.81 & 81.33 & 51.79 & 71.90 & 24.15 & 58.49 & 10.00 & 59.48 & 8.57 \\
    {} & ODIN & 98.86 & 95.16 & 71.72 & 40.65 & 68.52 & 23.05 & 58.46 & 10.77 & 59.38 & 9.52 \\
    \hline
    \multirow{2}{*}{MUC} & Original & 92.27 & 74.90 & 79.49 & 52.85 & 66.20 & 14.19 & 68.42 & 33.57 & 62.63 & 17.39 \\
    {} & ODIN & 92.67 & 75.71 & 79.54 & 53.22 & 65.72 & 14.11 & 68.32 & 33.45 & 62.17 & 17.27 \\
    \hline
    \multirow{2}{*}{PASS} & Original & 98.74 & 76.58 & 66.51 & 47.34 & 70.26 & 24.99 & 65.18 & 28.40 & 63.27 & 19.07 \\
    {} & ODIN & 90.40 & 74.33 & 63.08 & 35.20 & 69.81 & 21.83 & 65.93 & 29.03 & 62.73 & 17.78 \\
    \hline
    \multirow{2}{*}{LwF} & Original & 99.19 & 85.46 & 89.39 & 54.67 & 89.84 & 44.33 & 78.20 & 32.17 & 79.43 & 24.28 \\
    {} & ODIN & 98.52 & 90.39 & 88.94 & 63.04 & 88.68 & 47.56 & 76.83 & 36.20 & 77.02 & 28.29 \\
    \hline
    \multirow{2}{*}{BiC} & Original & 99.40 & 94.11 & 90.89 & 61.41 & 89.46 & 48.92 & 80.17 & 41.75 & 80.37 & 33.77 \\
    {} & ODIN & 98.57 & 95.14 & 91.86 & 64.29 & 87.89 & 47.40 & 74.54 & 37.40 & 76.27 & 29.06 \\
    \hline
    \multirow{2}{*}{DER++} & Original & 99.78 & 95.29 & 90.16 & 66.04 & 85.44 & 46.59 & 71.80 & 35.80 & 72.41 & 30.49 \\
    {} & ODIN & 99.09 & 94.96 & 87.08 & 63.07 & 87.72 & 49.26 & 73.92 & 37.87 & 72.91 & 32.52 \\
    \hline
    \multirow{2}{*}{HAT} & Original & 94.46 & 81.86 & 82.47 & 62.67 & 75.35 & 25.64 & 72.28 & 38.46 & 71.82 & 29.78 \\
    {} & ODIN & 94.56 & 82.06 & 82.45 & 62.60 & 75.36 & 25.84 & 72.31 & 38.61 & 71.83 & 30.01 \\
    \hline
    \multirow{2}{*}{HyperNet} & Original & 85.83 & 56.55 & 78.54 & 53.40 & 72.04 & 18.67 & 54.58 & 7.91 & 55.37 & 5.32 \\
    {} & ODIN & 86.89 & 64.31 & 79.39 & 56.72 & 73.89 & 23.8 & 54.60 & 8.64 & 55.53 & 6.91 \\
    \hline
    \multirow{2}{*}{Sup} & Original & 90.70 & 70.06 & 79.16 & 62.37 & 81.14 & 34.70 & 74.13 & 41.82 & 74.59 & 36.46 \\
    {} & ODIN & 90.68 & 69.70 & 82.38 & 62.63 & 81.48 & 36.35 & 73.96 & 41.10 & 74.61 & 36.46 \\
    \bottomrule
    \end{tabular}
    }
    \label{Tab:odin_additional}
\end{table}

A continual learning method with a better AUC shows a better CIL performance than other methods with lower AUC. 
For instance, original HAT achieves AUC of 82.47 while HyperNet achieves 78.54 on C10-5T. The CIL for HAT is 62.67 while it is 53.40 for HyperNet. However, there are some exceptions that this comparison does not hold. An example is LwF. Its AUC and CIL are 89.39 and 54.67 on C10-5T. Although its AUC is better than HAT, the CIL is lower. This is due to the fact that CIL improves with WP and TP according to Theorem~\ref{thm:ce}. The contraposition of Theorem~\ref{thm:necessary_condition} also says 
if the cross-entropy of TIL is large, that of CIL is also large. Indeed, the average within-task prediction (WP) accuracy for LwF on C10-5T is 95.2 while the same for HAT is 96.7. Improving WP is also important in achieving good CIL performances.

For PASS, we had to tune $\tau_k$ using a validation set. This is because the softmax in Eq.~\ref{eq:odin_softmax} 
improves AUC by making the IND (in-distribution) and OOD scores more separable within a task, but deteriorates the final scores across tasks. To be specific, the test instances are predicted as one of the classes in the first task after softmax because the relative values between classes in task 1 is larger than the other tasks in PASS. Therefore, larger $\tau_1$ and smaller $\tau_k$, for $k > 1$, are chosen to compensate the relative values.

\section{Definitions of TP} \label{apx:diff_tp}
As noted in the main paper, the class prediction in Eq.~\ref{eq:cil_in_til_and_tp} varies by definition of WP and TP. The precise definition of WP and TP depends on implementation. Due to this subjectivity, we follow the prediction method as the existing methods in continual learning, which is the $\argmax$ over the output. In this section, we show that the $\argmax$ over output is a special case of Eq.~\ref{eq:cil_in_til_and_tp}. We also provide CIL results using different definitions of TP.

We first establish another theorem. This is an extension of Theorem~\ref{thm:tp_ood} and connects the standard prediction method to our analysis.
\begin{theorem}[Extension of Theorem~\ref{thm:tp_ood}]
\label{thm:tp_ood_extension}

i) If $H_{TP} (x) \leq \delta$, let $\mathbf{P}'_k (x \in \mathbf{X}_k | D) = \mathbf{P} (x \in \mathbf{X}_k | D)^{1 / \tau_k}$, $\forall \tau_k > 0$, then $H_{OOD, k} (x) \leq \max (\delta /\tau_k, - \log (1 - (1 - e^{-\delta})^{1 / \tau_k}), \forall\, k = 1, \dots, T$. 

ii) If $H_{OOD, k} (x) \leq \delta_k, k=1,\dots,T$, let $\mathbf{P} (x \in \mathbf{X}_k | D) = \frac{\mathbf{P}_k' (x \in \mathbf{X}_k |D)^{1 / \tau_k}}{\sum_j \mathbf{P}_j' (x \in \mathbf{X}_j |D)^{1 / \tau_j}}$, $\forall \tau_k > 0$, then $H_{TP} (x) \leq \sum_k \frac{\mathbf{1}_{x \in \mathbf{X}_k} \delta_{k}}{\tau_{k}} + \frac{\sum_{k} (1 - e^{-\delta_k})^{1 / \tau_k}}{\sum_k \mathbf{1}_{x \in \mathbf{X}_k} (1 - (1 - e^{-\delta_{k}})^{1 / \tau_{k}})}$, where $\mathbf{1}_{x \in \mathbf{X}_k}$ is an indicator function.
\end{theorem}

In Theorem \ref{thm:tp_ood_extension} (proof appears later), we can observe that $\delta / \tau_k$ decreases with the increase of $\tau_k$, while $- \log (1 - (1 - e^{-\delta})^{1 / \tau_k})$ increases. Hence, when TP is given, let $\delta = H_{TP} (x)$, we can find the optimal $\tau_i$ to define OOD by solving $\delta /\tau_k = - \log (1 - (1 - e^{-\delta})^{1 / \tau_k})$. 
Similarly, given OOD, let $\delta_k = H_{OOD, k} (x)$, we can find the optimal $\tau_1, \dots, \tau_T$ to define TP by finding the global minima of $\sum_k \frac{\mathbf{1}_{x \in \mathbf{X}_k} \delta_{k}}{\tau_{k}} + \frac{\sum_{k} (1 - e^{-\delta_k})^{1 / \tau_k}}{\sum_k \mathbf{1}_{x \in \mathbf{X}_k} (1 - (1 - e^{-\delta_{k}})^{1 / \tau_{k}})}$. The optimal $\tau_k$ can be found using a memory buffer to save a small number of previous data like that in a replay-based continual learning method.

In Theorem~\ref{thm:tp_ood_extension} (ii), let $\mathbf{P}'_k (x \in \mathbf{X}_k |D) = \sigma ( \max f(x)_k)$, where $\sigma$ is the sigmoid and $f(x)_k$ is the output of task $k$ and choose $\tau_k \approx 0$  for each $k$. Then $\mathbf{P}(x \in \mathbf{X}_k | D)$ becomes approximately 1 for the task $k$ where the maximum logit value appears and 0 for the rest tasks. Therefore, Eq.~\ref{eq:cil_in_til_and_tp} in the paper
$$
\begin{aligned}
    \mathbf{P}(x \in \mathbf{X}_{k, j} | D) = \mathbf{P}(x \in \mathbf{X}_{k, j} | x \in \mathbf{X}_k, D) \mathbf{P}(x \in \mathbf{X}_k | D)
\end{aligned}
$$
is zero for all classes in tasks $k' \neq k$. Since only the probabilities of classes in task $k$ are non-zero, taking $\argmax$ over all class probabilities gives the same class as $\argmax$ over output logits.

We have also tried another definition of WP and TP. The considered WP is
\begin{align}
    \mathbf{P}(x \in \mathbf{X}_{k, j} | x \in \mathbf{X}_k, D) = \frac{e^{f(x)_{kj} / \nu_k }}{\sum_j e^{f(x)_{kj} / \nu_k }}, \label{eq:wip_max_softmax}
\end{align}
where $\nu_k$ is a temperature scaling parameter for task $k$, and the TP is
\begin{align}
    \mathbf{P}(x \in \mathbf{X}_k |D) = \frac{\mathbf{P}_k'(x \in \mathbf{X}_k | D) }{\sum_k \mathbf{P}_k' (x \in \mathbf{X}_k | D)}, \label{eq:tp_max_softmax}
\end{align}
where $\mathbf{P}_k'(x \in \mathbf{X}_k | D) =  \max_j e^{ f(x)_{kj} / \tau_k } / \sum_j e^{f(x)_{kj} / \tau_k }$
and $\tau_k$ is a temperature scaling parameter. 
This is the maximum softmax of task $k$.
We choose $\nu_k=0.1$ and $\tau_k=5$ for all $k$. A good $\tau$ and $\nu$ can be found using grid search on a validation set. However, one can also find the optimal values by optimization {using some past data saved for memory buffer.} The CIL results for the new prediction method is in Tab.~\ref{Tab:maintable_diff}.
\begin{table}[t]
\centering
\caption{
Average classification accuracy. The results are based on class prediction method defined with WP and TP in Eq.~\ref{eq:wip_max_softmax} and Eq.~\ref{eq:tp_max_softmax}, respectively. The results can improve by finding optimal temperature scaling parameters.
}
\begin{tabular}{l c c c c c c}
\toprule
\multirow{1}{*}{Method}  &  \multicolumn{1}{c}{M-5T} & \multicolumn{1}{c}{C10-5T}  &  \multicolumn{1}{c}{C100-10T} &  \multicolumn{1}{c}{C100-20T} &  \multicolumn{1}{c}{T-5T} & \multicolumn{1}{c}{T-10T} \\
\midrule
\textit{OWM} & 95.1\scalebox{1.0}{$\pm$0.11} & 40.6\scalebox{1.0}{$\pm$0.47} & 28.6\scalebox{1.0}{$\pm$0.82} & 22.9\scalebox{1.0}{$\pm$0.32} & 10.4\scalebox{1.0}{$\pm$0.54} & 9.2\scalebox{1.0}{$\pm$0.35} \Tstrut \\
\textit{MUC} & 75.7\scalebox{1.0}{$\pm$0.51} & 53.2\scalebox{1.0}{$\pm$1.32} & 30.6\scalebox{1.0}{$\pm$1.21} & 14.0\scalebox{1.0}{$\pm$0.12} & 33.1\scalebox{1.0}{$\pm$0.18} & 17.2\scalebox{1.0}{$\pm$0.13} \\
\textit{PASS}$^{\dagger}$ & 64.5\scalebox{1.0}{$\pm$2.64} & 33.6\scalebox{1.0}{$\pm$0.71} & 18.5\scalebox{1.0}{$\pm$1.85} & 20.8\scalebox{1.0}{$\pm$0.85} & 21.4\scalebox{1.0}{$\pm$0.44} & 13.0\scalebox{1.0}{$\pm$0.55} \\
LwF & 90.4\scalebox{1.0}{$\pm$1.18} & 63.0\scalebox{1.0}{$\pm$0.34} & 51.9\scalebox{1.0}{$\pm$0.88} & 47.5\scalebox{1.0}{$\pm$0.62} & 35.9\scalebox{1.0}{$\pm$0.32} & 27.8\scalebox{1.0}{$\pm$0.29} \\
iCaRL$^*$ & 87.4\scalebox{1.0}{$\pm$4.89} & 65.3\scalebox{1.0}{$\pm$0.83} & 52.9\scalebox{1.0}{$\pm$0.39} & 48.2\scalebox{1.0}{$\pm$0.70} & 34.8\scalebox{1.0}{$\pm$0.34} & 27.3\scalebox{1.0}{$\pm$0.17} \\
Mnemonics$^{\dagger *}$ & 91.8\scalebox{1.0}{$\pm$1.03} & 65.6\scalebox{1.0}{$\pm$1.55} & 50.7\scalebox{1.0}{$\pm$0.72} & 47.9\scalebox{1.0}{$\pm$0.71} & 36.3\scalebox{1.0}{$\pm$0.30} & 27.7\scalebox{1.0}{$\pm$0.78} \\
BiC & 95.1\scalebox{1.0}{$\pm$0.47} & 65.5\scalebox{1.0}{$\pm$0.81} & 50.8\scalebox{1.0}{$\pm$0.69} & 47.2\scalebox{1.0}{$\pm$0.71} & 37.0\scalebox{1.0}{$\pm$0.58} & 29.1\scalebox{1.0}{$\pm$0.34} \\
DER++ & 94.9\scalebox{1.0}{$\pm$0.50} & 63.1\scalebox{1.0}{$\pm$1.12} & 54.6\scalebox{1.0}{$\pm$1.21} & 48.9\scalebox{1.0}{$\pm$1.18} & 37.4\scalebox{1.0}{$\pm$0.72} & 32.1\scalebox{1.0}{$\pm$0.44} \\
\textit{HAT} & 82.1\scalebox{1.0}{$\pm$3.77} & 62.6\scalebox{1.0}{$\pm$1.31} & 41.5\scalebox{1.0}{$\pm$0.80} & 25.9\scalebox{1.0}{$\pm$0.56} & 38.9\scalebox{1.0}{$\pm$1.62} & 30.1\scalebox{1.0}{$\pm$0.52} \Tstrut \\
\textit{HyperNet} & 64.3\scalebox{1.0}{$\pm$2.98} & 56.7\scalebox{1.0}{$\pm$1.23} & 32.4\scalebox{1.0}{$\pm$1.07} & 24.5\scalebox{1.0}{$\pm$1.12} & 8.9\scalebox{1.0}{$\pm$0.58} & 7.0\scalebox{1.0}{$\pm$0.52} \\
\textit{Sup} & 69.7\scalebox{1.0}{$\pm$0.97} & 62.6\scalebox{1.0}{$\pm$1.11} & 46.8\scalebox{1.0}{$\pm$0.34} & 36.0\scalebox{1.0}{$\pm$0.32} & 41.5\scalebox{1.0}{$\pm$1.17} & 35.7\scalebox{1.0}{$\pm$0.40} \\
\hline
\textit{HAT+CSI} & 88.7\scalebox{1.0}{$\pm$1.27} & 85.2\scalebox{1.0}{$\pm$0.92} & 62.9\scalebox{1.0}{$\pm$1.07} & 53.6\scalebox{1.0}{$\pm$0.84} & 47.0\scalebox{1.0}{$\pm$0.38} & 46.2\scalebox{1.0}{$\pm$0.30} \\
\textit{Sup+CSI} & 64.9\scalebox{1.0}{$\pm$1.95} & 87.4\scalebox{1.0}{$\pm$0.40} & 66.6\scalebox{1.0}{$\pm$0.23} & 60.5\scalebox{1.0}{$\pm$0.89} & 47.7\scalebox{1.0}{$\pm$0.30} & 46.3\scalebox{1.0}{$\pm$0.30} \\
HAT+CSI+c & 93.4\scalebox{1.0}{$\pm$0.43} & 85.2\scalebox{1.0}{$\pm$0.94} & 63.6\scalebox{1.0}{$\pm$0.69} & 55.4\scalebox{1.0}{$\pm$0.79} & 51.4\scalebox{1.0}{$\pm$0.38} & 46.5\scalebox{1.0}{$\pm$0.26} \\
Sup+CSI+c & 62.2\scalebox{1.0}{$\pm$3.49} & 86.2\scalebox{1.0}{$\pm$0.79} & 67.0\scalebox{1.0}{$\pm$0.14} & 60.4\scalebox{1.0}{$\pm$1.04} & 48.2\scalebox{1.0}{$\pm$0.35} & 46.1\scalebox{1.0}{$\pm$0.32} \\ 
\bottomrule
\end{tabular}
\label{Tab:maintable_diff}
\end{table}

\begin{proof}
[Proof of Theorem \ref{thm:tp_ood_extension}.]
\label{prf:tp_ood_extension}
i) Assume $x \in \mathbf{X}_{k_0}$. 

For $k = k_0$, we have
$$
\begin{aligned}
    H_{OOD, k_0} (x) &= - \log \mathbf{P}'_{k_0} (x \in \mathbf{X}_{k_0} | D) \\ %^{1 / \tau_k} \\
    &= - \frac{1}{\tau_{k_0}} \log \mathbf{P}(x \in \mathbf{X}_{k_0} | D) \\
    &= \frac{1}{\tau_{k_0}} H_{TP} (x) \leq \frac{\delta}{\tau_{k_0}}.
\end{aligned}
$$
For $k \neq k_0$, we have
$$
\begin{aligned}
    H_{OOD, k} (x) &= - \log \mathbf{P}'_k (x \notin \mathbf{X}_k | D) \\
    &= - \log (1 - \mathbf{P}'_k (x \in \mathbf{X}_k | D))\\
    &= - \log (1 - \mathbf{P} (x \in \mathbf{X}_k | D)^{1 / \tau_k})\\
    &= - \log (1 - (1 - \mathbf{P}(x \in \cup_{k' \neq k} \mathbf{X}_{k'} | D))^{1 / \tau_k})\\ 
    &\leq - \log (1 - (1 - \mathbf{P}(x \in \mathbf{X}_{k_0} | D))^{1 / \tau_k})\\
    &= - \log (1 - (1 - e^{-H_{TP}(x)})^{1 / \tau_k})\\
    &\leq - \log (1 - (1 - e^{-\delta})^{1 / \tau_k}).\\
\end{aligned}
$$

ii) Assume $x \in \mathbf{X}_{k_0}$. 

For $k = k_0$, by $H_{OOD, {k_0}} (x) \leq \delta_{k_0}$, we have 
$$- \log \mathbf{P}'_{k_0} (x \in \mathbf{X}_{k_0} | D) \leq \delta_{k_0},$$ 
which means 
$$\mathbf{P}'_{k_0} (x \in \mathbf{X}_{k_0} | D) \geq e^{-\delta_{k_0}}.$$
For $k \neq k_0$, by $H_{OOD, {k}} (x) \leq \delta_{k}$, we have 
$$- \log \mathbf{P}'_{k} (x \notin \mathbf{X}_{k} | D) \leq \delta_{k},$$
which means 
$$\mathbf{P}'_{k} (x \in \mathbf{X}_{k} | D) \leq 1 - e^{-\delta_{k}}.$$

Therefore, we have
$$
\begin{aligned}
    \mathbf{P} (x \in \mathbf{X}_{k_0} | D) &= \frac{\mathbf{P}'_{k_0} (x \in \mathbf{X}_{k_0} |D)^{1/\tau_{k_0}}}{\sum_k \mathbf{P}'_k (x \in \mathbf{X}_k |D)^{1/\tau_{k}}} \\
    &\geq \frac{e^{-\delta_{k_0} / \tau_{k_0}}}{1 + \sum_{k\neq k_0} (1 - e^{-\delta_k})^{1 / \tau_k}} \\
    &= \frac{e^{-\delta_{k_0} / \tau_{k_0}}}{1 - (1 - e^{-\delta_{k_0}})^{1 / \tau_{k_0}} + \sum_{k} (1 - e^{-\delta_k})^{1 / \tau_k}} \\
    &= \frac{e^{-\delta_{k_0} / \tau_{k_0}}}{1 - (1 - e^{-\delta_{k_0}})^{1 / \tau_{k_0}}} \cdot
    \frac{1}{1 + \frac{\sum_{k} (1 - e^{-\delta_k})^{1 / \tau_k}}{1 - (1 - e^{-\delta_{k_0}})^{1 / \tau_{k_0}}}}.
\end{aligned}
$$

Hence, 
$$
\begin{aligned}
    H_{TP} (x) &= -\log \mathbf{P} (x \in \mathbf{X}_{k_0} | D) \\
    &\leq - \log \frac{e^{-\delta_{k_0} / \tau_{k_0}}}{1 - (1 - e^{-\delta_{k_0}})^{1 / \tau_{k_0}}} \cdot
    \frac{1}{1 + \frac{\sum_{k} (1 - e^{-\delta_k})^{1 / \tau_k}}{1 - (1 - e^{-\delta_{k_0}})^{1 / \tau_{k_0}}}} \\
    &= \frac{\delta_{k_0}}{\tau_{k_0}} + \log [1 - (1 - e^{-\delta_{k_0}})^{1 / \tau_{k_0}}] 
    + \log \left[1 + \frac{\sum_{k} (1 - e^{-\delta_k})^{1 / \tau_k}}{1 - (1 - e^{-\delta_{k_0}})^{1 / \tau_{k_0}}}\right]\\
    &\leq \frac{\delta_{k_0}}{\tau_{k_0}} + \frac{\sum_{k} (1 - e^{-\delta_k})^{1 / \tau_k}}{1 - (1 - e^{-\delta_{k_0}})^{1 / \tau_{k_0}}} \\
    &= \sum_k \frac{\mathbf{1}_{x \in \mathbf{X}_k} \delta_{k}}{\tau_{k}} + \frac{\sum_{k} (1 - e^{-\delta_k})^{1 / \tau_k}}{\sum_k \mathbf{1}_{x \in \mathbf{X}_k} (1 - (1 - e^{-\delta_{k}})^{1 / \tau_{k}})}.
\end{aligned}
$$
\end{proof}

\section{Details of HAT, Sup, and CSI} \label{apx:csi}
We have proposed two highly effective new CIL methods, HAT+CSI and Sup+CSI, by integrating the existing parameter isolation based continual learning (CL) method HAT~\cite{Serra2018overcoming} or Sup~\cite{supsup2020} with the strong OOD detection method CSI~\cite{tack2020csi}. We replaced the training loss of HAT and Sup by that of CSI while applying the continual learning techniques of the respective method. In this section, we overview Sup, HAT, and CSI, and explain how to train them continually. Figure~\ref{fig:diagrams} shows the overall training frameworks of Sup+CSI and HAT+CSI.

Denote feature extractor by $h$, classifier by $f$, and the parameters by $\mathbf{W}$. In the main paper, we denote the output of task $k$ by $f(x)_k$ for both a single-head or multi-head method (e.g., Eq.~\ref{eq:cil_pred}) for consistency. In this section, we use $f(x, k)$ to indicate the output of task $k$ to be more explicit as both HAT and Sup are multi-head methods (one head for each task) designed for task incremental learning (TIL).

\subsection{Sup}
\begin{figure*}
\centering
\subfigure[]{\includegraphics[width=68mm]{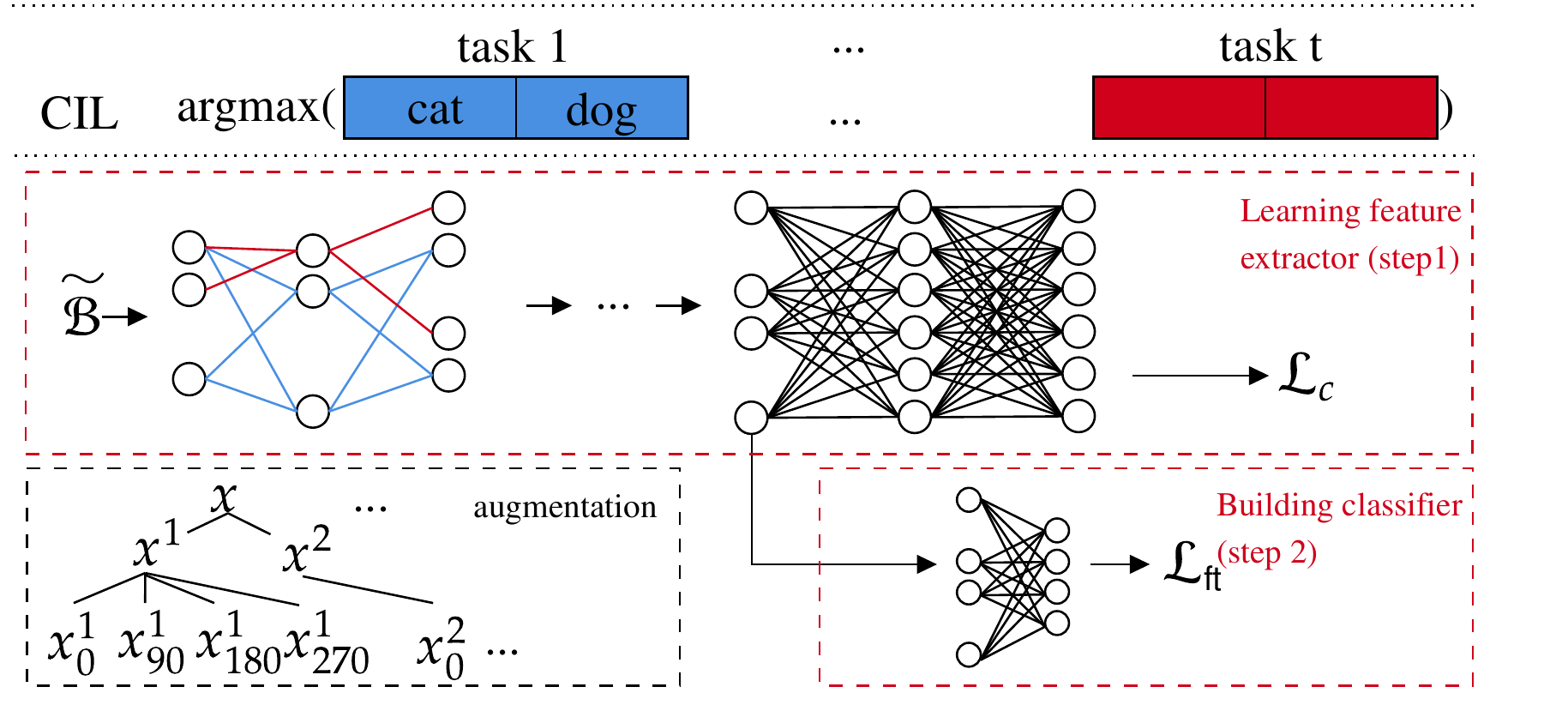}}
\subfigure[]{\includegraphics[width=68mm]{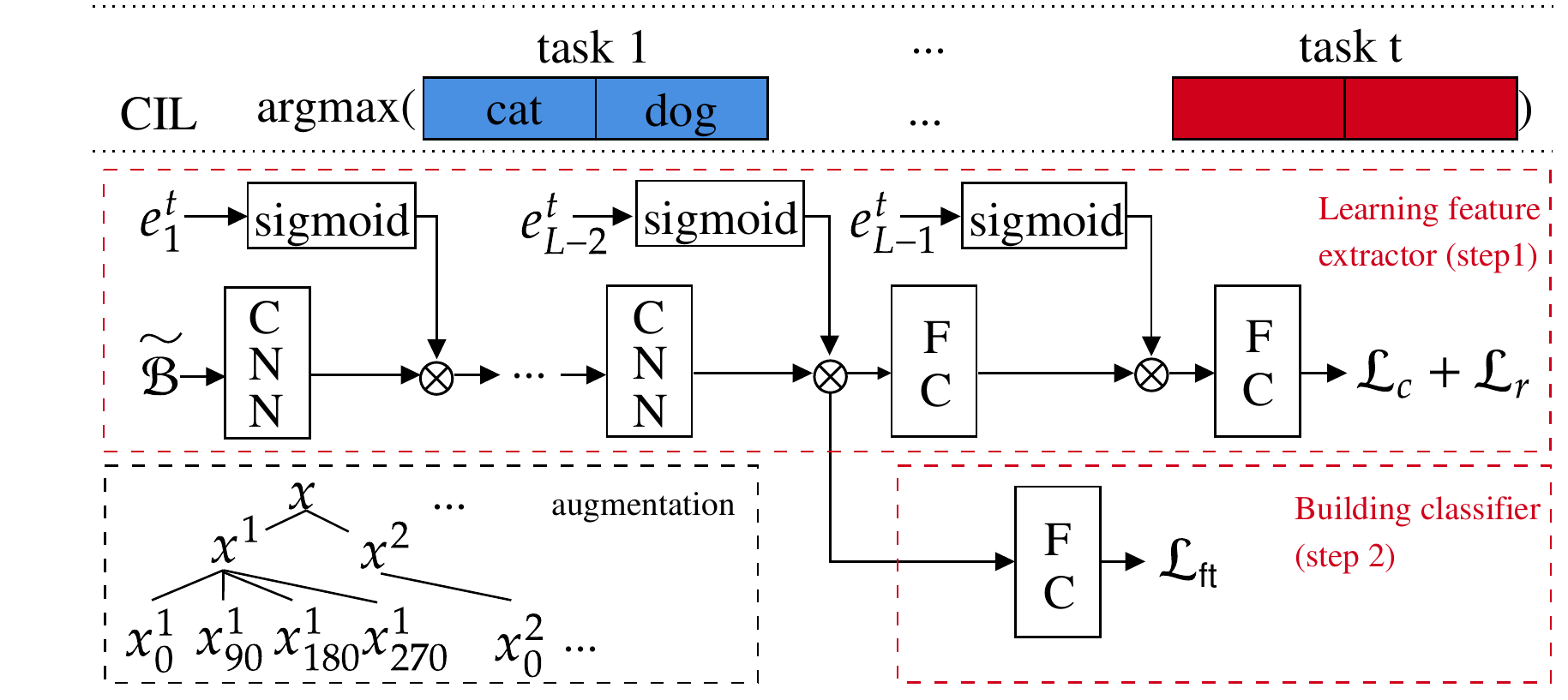}}
\caption{
Overview of prediction and training framework of Sup+CSI and HAT+CSI. (a) Sup+CSI: The CIL prediction is made by taking argmax over the concatenated output values from each task. In training the network, the training batch is augmented to give different views of samples for contrastive training in the OOD detection algorithm CSI. The training consists of two steps following CSI. The first step is learning the feature extractor. In this step, the Edge Popup algorithm~\cite{ramanujan2020s} is applied to find a sparse network for each task. The sparse networks, which are indicated by edges of different colors in the diagram. The second step fine-tunes the classifier only the fixed feature extractor. (b) HAT+CSI: The CIL prediction is also made by argmax over the concatenated output from each task as Sup+CSI method. Due to the OOD detection algorithm CSI, the overall training process is similar to Sup+CSI except that it applies the hard attention algorithm~\cite{Serra2018overcoming}.
In training feature extractor, task embeddings are applied to find hard masks at each layer. Then given the learned feature representations, fine-tunes the classifier in step 2.
}
\label{fig:diagrams}
\end{figure*}

SupSup (Sup)~\cite{supsup2020} trains supermasks by Edge Popup algorithm~\cite{ramanujan2020s}. More precisely, given initial $\mathbf{W}$, find binary masks $\mathbf{M}_k$ for task $k$ to minimize the cross-entropy loss
\begin{align}
    \mathcal{L} = - \frac{1}{|\mathbf{X}_{k}|} \sum \log p(y | x, k), \label{eq:loss_sup}
\end{align}
where $\mathbf{X}_k$ is the training data for task $k$, and
\begin{align}
    p(y | x, k) = f(h(x; \mathbf{W} \otimes \mathbf{M}_{k})),
\end{align}
where $\otimes$ indicates an element-wise product. The masks are obtained by selecting the top $p$\% of entries in the score matrices $\mathbf{V}$. The $p$ value determines the sparsity of the mask $\mathbf{M}_{k}$. The subnetwork found by Edge Popup algorithm is indicated by different colors in Figure~\ref{fig:diagrams}(a).

Given the task-id $k$ of a test instance at inference, the system (which is referred as Sup GG in the original Sup paper) uses the task-specific mask $\mathbf{M}_k$ to obtain the classification output. By integrating the OOD detection method, CSI, during training, Sup+CSI does not require to know the task-id of test instance, which makes Sup+CSI applicable to CIL (class incremental learning).

\subsection{HAT}
We now discuss the hard attention (mask) mechanism of HAT~\cite{Serra2018overcoming}. It finds binary masks ${a}_{l}^{k}$ for each layer $l$ and task $k$, and uses them to block/unblock information flow at forward and backward pass. More precisely, the hard attention is defined as
\begin{align}
    {a}_{l}^{k} = \sigma (s {e}_{l}^{k} ),
\end{align}
where $\sigma$ is the sigmoid, $s$ is a positive constant, and ${e}_{l}^{k}$ is a learnable embedding. To approximate the binary mask, the system uses a large $s$ value. The attention is applied to the output at each layer as
\begin{align}
    {h}_{l}^{'} = {a}_{l}^{k} \otimes {h}_{l},
\end{align}
where $\otimes$ is an element-wise product, and
\begin{align}
    {h}_{l} = \text{ReLU}({W}_{l} {h}_{l-1} + {b}_{l}).
\end{align}
The neurons with attention value $1$ is important for task $k$ while those with zero attention value are not necessary for the task, and thus they can be freely changed without affecting the output value ${h}_{l}^{'}$. 
The system needs to know which neurons are important to protect the previous knowledge from forgetting. Denote the accumulated attentions of all previous tasks by
\begin{align}
    {a}_{l}^{<k} = \max ( {a}_{l}^{< k-1}, {a}_{l}^{k-1} ),
\end{align}
where ${a}_{l}^{0}$ is the zero vector and $\max$ is an element-wise maximum. The gradients of parameters corresponding to important neurons is modified as
\begin{align}
    \nabla w_{ij, l}' = \left( 1 - \min \left( a_{i, l}^{< k}, a_{j, l-1}^{< k} \right) \right) \nabla w_{ij, l}, \label{hatupdate}
\end{align}
where $a_{i, l}^{< k}$ is the $i$'th unit of ${a}_{l}^{< k}$ and $l=1, \cdots, L-1$. The hard attention is not applied to the last layer $L$ since it is a task-specific classification layer.

To encourage sparsity in ${a}_{l}^{k}$, the system uses regularization as
\begin{align}
    \mathcal{L}_{r} = \lambda_{k} \frac{\sum_{l}\sum_{i} a_{i, l}^{k}\left( 1 - a_{i, l}^{< k} \right)}{\sum_{l}\sum_{i} \left( 1 - a_{i, l}^{< k} \right)},
\end{align}
where $\lambda_{k}$ is a hyper-parameter. The system minimizes the loss
\begin{align}
    \mathcal{L} = \mathcal{L}_{ce} + \mathcal{L}_{r}, \label{eq:loss_hat}
\end{align}
where $\mathcal{L}_{ce}$ is the cross-entropy loss. The overall framework of the algorithm is shown in Figure~\ref{fig:diagrams}(b).

\subsection{CSI}
We now explain the OOD detection method CSI, and how to incorporate it in HAT and Sup. CSI is based on contrastive learning~\cite{chen2020simple, he2020momentum} and data augmentation due to their excellent performance~\cite{tack2020csi}. 
Since this section focuses on how to learn a single task based on OOD detection, we omit the task-id unless necessary. The OOD training process is similar to that of contrastive learning. It consists of two steps: 1) learning {the feature representation by} the composite $g \circ h$, where $h$ is a feature extractor and $g$ is a projection to contrastive representation, and 2) learning a linear classifier $f$ mapping the feature representation of $h$ to the label space. This two step training process is outlined in Figure~\ref{fig:diagrams}(a) and (b). In the following, we describe the training process: contrastive learning for feature representation learning (1), and OOD classifier building (2). We then explain how to make a prediction based on an ensemble method to further improve prediction.

\subsubsection{Contrastive Loss for Feature Learning.}
This is step 1. Supervised contrastive learning is used to try to repel data of different classes and align data of the same class more closely to make it easier to classify them. A key operation is data augmentation via transformations. 

{Given a batch of $N$ samples, each sample ${x}$ is first duplicated and each version then goes through \textit{three initial augmentations}
(horizontal flip, color changes, and Inception crop~\cite{inception}) to generate two different views ${x}^{1}$ and ${x}^{2}$ (they keep the same class label as ${x}$).}
Denote the augmented batch by $\mathcal{B}$, which now has $2N$ samples. {In~\cite{hendrycks2019using,tack2020csi}}, it was shown that using image rotations is effective in learning OOD detection models because such rotations can effectively serve as out-of-distribution (OOD) training data.  
For each augmented sample ${x} \in \mathcal{B}$ with class $y$ of a task, we rotate ${x}$ by $90^{\circ}, 180^{\circ}, 270^{\circ}$ to create three images, which are assigned \textit{three new classes} $y_1, y_2$, and $y_3$, respectively.
{This results in a larger augmented batch $\tilde{\mathcal{B}}$. Since we generate three new images from each ${x}$, %\in \mathcal{B}$, 
the size of $\tilde{\mathcal{B}}$ is $8N$. For each original class, we now have 4 classes. For a sample ${x} \in \tilde{\mathcal{B}}$, let $\mathcal{\tilde{B}}({x}) = \mathcal{\tilde{B}} \backslash \{ {x} \}$
and 
let $P({x}) \subset \tilde{\mathcal{B}} \backslash \{ {x} \}$ 
be a set consisting of the data of the same class as ${x}$ distinct from ${x}$.
The contrastive representation of a sample ${x}$ is ${z}_{x} = g(h({x}, t)) / \| g(h({x}, t)) \|$, where $t$ is the current task. 
In learning, we minimize the supervised contrastive loss~\cite{khosla2020supervised} of task $t$. 
\begin{align}
    \mathcal{L}_{c}
    % {\Phi}
    % , {a}^{t}
    &= \frac{1}{8N} \sum_{ {x} \in \tilde{\mathcal{B}}} \frac{-1}{| P({x}) |} \sum_{{p} \in P({x})} \log{ \frac{ \text{exp}( {z}_{{x}} \cdot {z}_{{p}} / \tau)}{\sum_{{x}'  \in \tilde{\mathcal{B}}({x}) } \text{exp}( {z}_{{x}} \cdot {z}_{{x}'} / \tau) } }, \label{modsupclr}
\end{align}
where 
$\tau$ is a scalar temperature, $\cdot$ is dot product, and $\times$ is  multiplication. 
The loss is reduced by repelling ${z}$ of different classes and aligning ${z}$ of the same class more closely.
$\mathcal{L}_{c}$ basically trains a feature extractor with good 
representations for learning an  OOD classifier.} 

Since the feature extractor is shared across tasks in continual learning, a protection is needed to prevent catastrophic forgetting. HAT and Sup use their respective technique to protect their feature extractor from forgetting. Therefore, the losses $\mathcal{L}$ of Eq.~\ref{eq:loss_sup} and $\mathcal{L}_{ce}$ of Eq.~\ref{eq:loss_hat} are replaced by Eq.~\ref{modsupclr} while the forgetting prevention mechanisms still hold.

\subsubsection{Learning the Classifier.} 
This is step 2. {Given the feature extractor $h$ trained with the loss in Eq.~\ref{modsupclr}, we {\textit{freeze $h$} and} only \textit{fine-tune} the linear classifier $f$, which is trained to predict the classes of task $t$ \textit{and} the augmented rotation classes.} $f$ maps the feature representation to {the label space in} $\mathcal{R}^{4|\mathcal{C}^{t}|}$, where $4$ is the number of rotation classes including the original data with $0^{\circ}$ rotation and $|\mathcal{C}^{t}|$ is the number of {original} classes in task $t$. We minimize the cross-entropy loss,
\begin{align}
    \mathcal{L}_{\text{ft}} = - \frac{1}{|\tilde{\mathcal{B}} |} \sum_{({x}, y) \in \tilde{\mathcal{B}}} 
    \log \tilde{p}(y | {x}, t), 
    \label{3obj}
\end{align}
where $\text{ft}$ indicates fine-tune,
and 
\begin{align}
    \tilde{p}(y | {x}, t) = \text{softmax} \left( f(h({x}, t))
    \right) \label{probrotation}
\end{align}
where
$f(h({x}, t)) \in \mathcal{
R}^{4|\mathcal{C}^{t}|}$. The output $f(h({x}, t))$
includes the rotation classes. The linear classifier is trained to predict the original \textit{and} the rotation classes. Since individual classifier is trained for each task and the feature extractor is frozen, no protection is necessary.

\subsubsection{Ensemble Class Prediction.} \label{apx:ensemblesection}
We describe how to predict a label $y \in \mathcal{C}^{t}$ (TIL) and $y \in \mathcal{C}$ (CIL) ($\mathcal{C}$ is the set of original classes of all tasks).
We assume all tasks have been learned and their models are protected by masks.
% , which we discuss in the next subsection. 

We discuss the prediction of class label $y$ for a test sample ${x}$ in the TIL setting first. Note that the network $f\circ h$ in Eq.~\ref{probrotation} returns logits for rotation classes (including the original task classes). Note also for each original class label $j_k \in \mathcal{C}^{k}$ (original classes) of a task $k$, we created three additional rotation classes. For class $j_k$, the classifier $f$ will produce four output values from its four rotation class logits, i.e., $f_{j_k,0}(h({x_0}, k))$, $f_{j_k,90}(h({x_{90}}, k))$, $f_{j_k,180}(h({x_{180}}, k))$, and $f_{j_k,270}(h({x_{270}}, k))$, where 0, 90, 180, and 270 represent $0^{\circ}, 90^{\circ}, 180^{\circ}$, and $270^{\circ}$ rotations respectively and ${x}_0$ is the original ${x}$. 
We compute an ensemble output $f_{j_k}(h({x},k))$ for each class $j_k \in \mathcal{C}^{k}$ of task $k$, 
\begin{align}
    f(h({x},k))_{j_k} = \frac{1}{4} \sum_{\text{deg}} f (h({x}_{\text{deg}}, k))_{j_k,\text{deg}} \label{eq:ensemblelogit}.
\end{align}
We use Eq.~\ref{eq:cil_pred} to make the CIL class prediction, where the final class prediction is made as 
\begin{align}
    \hat{y} = \argmax \bigoplus_{i} f(h(x, i)). \label{eq:afterensemble}
\end{align}

\section{Output Calibration} \label{apx:calibration}
In this section, we discuss the output calibration technique used in Sec.~\ref{sec.HAT+CSI} to improve the final prediction accuracy. Even if an OOD detection of each task was perfect (i.e. the model accept and reject IND and OOD samples perfectly), the system could make incorrect class prediction if the magnitudes of outputs across different tasks are different. To ensure that the output values are comparable, we calibrate the outputs by scaling $\alpha_k$ and shifting $\beta_k$ for each task. The optimal parameters $(\alpha_k, \beta_k) \in R \times R$ can be found by solving optimization problem using samples in memory buffer. More precisely, denote the memory buffer $\mathcal{M}$ and calibration parameters $( \alpha, \beta ) \in R^{T} \times R^{T}$, where $T$ is the number of learned tasks. After training $T$th task, we find optimal calibration parameters by minimizing the cross-entropy loss,
\begin{align}
    \mathcal{L} = - \frac{1}{|\mathcal{M}|} \sum_{(x, y) \in \mathcal{M}} \log p(y | x)
\end{align}
where $p(c | x)$ is computed using the softmax,
\begin{align}
    \text{softmax} \bigoplus [ \alpha_k f(x)_k + \beta_k ]
\end{align}
where $\bigoplus$ indicates the concatenation and $f(x)_k$ is the output of task $k$ as Eq.~\ref{eq:cil_pred}. Given the optimal parameters $(\alpha^*, \beta^*)$, we make final prediction as
\begin{align}
    \hat{y} = \argmax \bigoplus [ \alpha_k^* f(x)_k + \beta_k^* ]
\end{align}
If we use $OOD_k = \sigma ( \alpha_k^* f(x)_k + \beta_k^* )$, where $\sigma$ is the sigmoid, and $TP_k = OOD_k / \sum_{k'} OOD_{k'}$, the theoretical results in Sec.~\ref{sec.theorem} hold.

\section{TIL (WP) Results} \label{apx:til_results}
{\color{black}
The TIL (WP) results of all the systems are reported in Tab.~\ref{Tab:til_full}. HAT and Sup show strong performances compared to the other baselines as they leverage task-specific parameters. However, as shown in Theorem~\ref{thm:ce}, the CIL depends on TP (or OOD). Without an OOD detection mechanism in HAT or Sup, they perform poorly in CIL as shown in the main paper. The contrastive learning in CSI also improves the IND prediction (i.e., WP), and this along with OOD detection results in the strong CIL performance.
}
\begin{table}[t]
\centering
\caption{
{The TIL results of all the systems. The calibrated versions (+c) of our methods are omitted as calibration does not affect TIL performance. Exemplar-free methods are italicized.}
}
% \resizebox{1.0\columnwidth}{!}{
\begin{tabular}{l c c c c c c}
\toprule
\multirow{1}{*}{Method}  &  \multicolumn{1}{c}{M-5T} & \multicolumn{1}{c}{C10-5T}  &  \multicolumn{1}{c}{C100-10T} &  \multicolumn{1}{c}{C100-20T} &  \multicolumn{1}{c}{T-5T} & \multicolumn{1}{c}{T-10T} \\
\midrule
\textit{OWM} & 99.7\scalebox{0.9}{$\pm$0.03} & 85.0\scalebox{0.9}{$\pm$0.07} & 59.6\scalebox{0.9}{$\pm$0.83} & 65.4\scalebox{0.9}{$\pm$0.48} & 22.4\scalebox{0.9}{$\pm$0.87} & 28.1\scalebox{0.9}{$\pm$0.55} \\
\textit{MUC} & 99.9\scalebox{0.9}{$\pm$0.02} & 95.1\scalebox{0.9}{$\pm$0.10} & 77.3\scalebox{0.9}{$\pm$0.83} & 73.4\scalebox{0.9}{$\pm$9.16} & 55.9\scalebox{0.9}{$\pm$0.26} & 47.2\scalebox{0.9}{$\pm$0.22} \\
\textit{PASS}$^{\dagger}$ & 99.5\scalebox{0.9}{$\pm$0.14} & 83.8\scalebox{0.9}{$\pm$0.68} & 72.1\scalebox{0.9}{$\pm$0.70} & 76.8\scalebox{0.9}{$\pm$0.32} & 49.9\scalebox{0.9}{$\pm$0.56} & 46.5\scalebox{0.9}{$\pm$0.39} \\
LwF & 99.9\scalebox{0.9}{$\pm$0.09} & 95.2\scalebox{0.9}{$\pm$0.30} & 86.2\scalebox{0.9}{$\pm$1.00} & 89.0\scalebox{0.9}{$\pm$0.45} & 56.4\scalebox{0.9}{$\pm$0.48} & 55.3\scalebox{0.9}{$\pm$0.35} \\
iCaRL & 99.9\scalebox{0.9}{$\pm$0.08} & 94.9\scalebox{0.9}{$\pm$0.34} & 84.2\scalebox{0.9}{$\pm$1.04} & 85.7\scalebox{0.9}{$\pm$0.68} & 54.5\scalebox{0.9}{$\pm$0.29} & 52.7\scalebox{0.9}{$\pm$0.37} \\
Mnemonics$^{\dagger *}$ & 99.9\scalebox{0.9}{$\pm$0.03} & 94.5\scalebox{0.9}{$\pm$0.46} & 82.3\scalebox{0.9}{$\pm$0.30} & 86.2\scalebox{0.9}{$\pm$0.46} & 54.8\scalebox{0.9}{$\pm$0.16} & 52.9\scalebox{0.9}{$\pm$0.66} \\
BiC & 99.9\scalebox{0.9}{$\pm$0.03} & 95.4\scalebox{0.9}{$\pm$0.35} & 84.6\scalebox{0.9}{$\pm$0.48} & 88.7\scalebox{0.9}{$\pm$0.19} & 61.5\scalebox{0.9}{$\pm$0.60} & 62.2\scalebox{0.9}{$\pm$0.45} \\
DER++ & 99.7\scalebox{0.9}{$\pm$0.08} & 92.0\scalebox{0.9}{$\pm$0.54} & 84.0\scalebox{0.9}{$\pm$9.43} & 86.6\scalebox{0.9}{$\pm$9.44} & 57.4\scalebox{0.9}{$\pm$1.31} & 60.0\scalebox{0.9}{$\pm$0.74} \\
\textit{HAT} & 99.9\scalebox{0.9}{$\pm$0.02} & 96.7\scalebox{0.9}{$\pm$0.18} & 84.0\scalebox{0.9}{$\pm$0.23} & 85.0\scalebox{0.9}{$\pm$0.98} & 61.2\scalebox{0.9}{$\pm$0.72} & 63.8\scalebox{0.9}{$\pm$0.41} \\
\textit{HyperNet} & 99.7\scalebox{0.9}{$\pm$0.04} & 94.6\scalebox{0.9}{$\pm$0.37} & 76.8\scalebox{0.9}{$\pm$1.22} & 83.5\scalebox{0.9}{$\pm$0.98} & 23.9\scalebox{0.9}{$\pm$0.60} & 28.0\scalebox{0.9}{$\pm$0.69} \\
\textit{Sup} & 99.6\scalebox{0.9}{$\pm$0.01} & 96.6\scalebox{0.9}{$\pm$0.21} & 87.9\scalebox{0.9}{$\pm$0.27} & 91.6\scalebox{0.9}{$\pm$0.15} & 64.3\scalebox{0.9}{$\pm$0.24} & 68.4\scalebox{0.9}{$\pm$0.22} \\
\hline
\textit{HAT+CSI} & 99.9\scalebox{0.9}{$\pm$0.00} & 98.7\scalebox{0.9}{$\pm$0.06} & 92.0\scalebox{0.9}{$\pm$0.37} & 94.3\scalebox{0.9}{$\pm$0.06} & 68.4\scalebox{0.9}{$\pm$0.16} & 72.4\scalebox{0.9}{$\pm$0.21} \\
\textit{Sup+CSI} & 99.0\scalebox{0.9}{$\pm$0.08} & 98.7\scalebox{0.9}{$\pm$0.07} & 93.0\scalebox{0.9}{$\pm$0.13} & 95.3\scalebox{0.9}{$\pm$0.20} & 65.9\scalebox{0.9}{$\pm$0.25} & 74.1\scalebox{0.9}{$\pm$0.28} \\
\bottomrule
\end{tabular}
\label{Tab:til_full}
\end{table}

\section{Hyper-parameters}\label{apx:hyper_params}
Here we report the hyper-parameters that we did not report in the main paper due to space limitations. We mainly report the hyper-parameters of the proposed methods, HAT+CSI, Sup+CSI, and their calibrated versions. For all the experiments of the proposed methods, {we use the values chosen by the original CSI~\cite{tack2020csi}.}
We use LARS~\cite{you2017large} optimization with learning rate 0.1 for training the feature extractor. We linearly increase the learning rate by 0.1 per epoch for the first 10 epochs. After that, we use cosine scheduler~\cite{loshchilov2016sgdr} without restart as in \cite{tack2020csi,chen2020simple}. After training the feature extractor, we train the linear classifier for 100 epochs with SGD with learning rate 0.1 and reduce the rate by 0.1 at 60, 75, and 90 epochs. For all the experiments except MNIST, we train the feature extractor for 700 epochs with batch size 128. 

For the following hyper-parameters, we use 10\% of training data for validation to find a good set of values. For the number of epochs and batch size for MNIST, Sup+CSI trains for 1000 epochs with batch size of 32 while HAT+CSI trains for 700 epochs with batch size of 256.
The hard attention regularization penalty $\lambda_i$ in HAT is different by experiments and task $i$. For MNIST, we use $\lambda_1 = 0.25$, and $\lambda_2 = \cdots = \lambda_5 = 0.1$. For C10-5T, we use $\lambda_1=1.0$, and $\lambda_2 = \cdots = \lambda_5 = 0.75$. For C100-10T, $\lambda_1=1.5$, and $\lambda_2 = \cdots = \lambda_{10} = 1.0$ are used. For C100-20T, $\lambda_1 = 3.5$, and $\lambda_2 = \cdots = \lambda_{20} = 2.5$ are used. For T-5T, $\lambda_i = 0.75$ for all tasks, and lastly, for T-10T, $\lambda_1 = 1.0$, and $\lambda_2 = \cdots = \lambda_{10} = 0.75$ are used. We use larger $\lambda_1$ for the first task than the later tasks as we have found that the larger regularization on the first task results in better accuracy. This is by the definition of regularization in HAT. The earlier task gives lower penalty than later tasks. We manually give larger penalty to the first task. We did not search hyper-parameter $\lambda_t$ for tasks $t \geq 2$. For sparsity in Sup+CSI, we simply choose the least sparsity value of 32 used in the original Sup paper without parameter search.

Calibration methods (HAT+CSI+c and Sup+CSI+c) are based on its memory free versions (i.e. HAT+CSI and Sup+CSI). Therefore, the model training part uses the same hyper-parameters as their calibration free counterparts. 
{For calibration training, 
we use SGD with learning rate 0.01, 160 training iterations, and batch size of 15 for HAT+CSI+c for all experiments. For Sup+CSI+c, we use the same values for all the experiments except for MNIST.
For MNIST,
we use learning rate 0.05, batch size of 8, and run 280 iterations.
}

For the baselines, we use the hyper-parameters reported in the original papers or in their code. If the hyper-parameters are unknown or the code does not reproduce the result (e.g., the baseline did not implement a particular dataset or the code had undergone significant version change), we search for the hyper-parameters as we did for HAT+CSI and Sup+CSI.

\section{Computes and Resources Used in Experiments} \label{apx:n_params}
\begin{table}
\centering
\caption{
The number of parameters used at inference after learning the final task. The M after each value indicates millions.
}
\begin{tabular}{ c c c c c c c c } 
\toprule
Method & M-5T & C10-5T & C100-10T & C100-20T & T-5T & T-10T \\ 
\midrule
OWM & 5.27M & 5.27M & 5.36M & 5.36M & 5.46M & 5.46M \\
MUC & 1.06M & 11.19M & 45.06M & 45.06M & 45.47M & 45.47M \\
PASS & 1.03M & 11.17M & 44.76M & 44.76M & 44.86M & 44.86M \\
LwF & 1.03M & 11.17M & 44.76M & 44.76M & 44.86M & 44.86M \\
iCaRL & 1.03M & 11.17M & 44.76M & 44.76M & 44.86M & 44.86M \\
Mnemonics & 1.03M & 11.17M & 44.76M & 44.76M & 44.86M & 44.86M \\
BiC & 1.03M & 11.17M & 44.76M & 44.76M & 44.86M & 44.86M \\
DER++ & 1.03M & 11.17M & 44.76M & 44.76M & 44.86M & 44.86M \\
HAT & 1.04M & 11.23M & 45.01M & 45.28M & 44.97M & 45.11M \\
HyperNet & 0.48M & 0.47M & 0.47M & 0.47M & 0.48M & 0.48M \\
Sup & 0.05M & 1.43M & 5.75M & 11.45M & 2.95M & 5.80M \\
\hline
HAT+CSI & 1.07M & 11.25M & 45.31M & 45.58M & 45.59M & 45.72M \\ 
HAT+CSI+c & 1.07M & 11.25M & 45.31M & 45.58M & 45.59M & 45.72M \\
Sup+CSI & 0.28M & 1.38M & 5.90M & 11.60M & 3.04M & 6.05M \\ 
Sup+CSI+c & 0.28M & 1.38M & 5.90M & 11.60M & 3.04M & 6.05M \\
\bottomrule
\end{tabular}
\end{table}
This paper provides a guidance on how to solve the CIL problem, backed by theoretical justifications. Based on the guidance, we have proposed some new CIL methods. Two outstanding ones are HAT+CSI and Sup+CSI. These methods achieve state-of-the-art CIL performances, but by no mean, they are the only approaches. Many CIL algorithms can be designed following the analysis as it is general to any CL model.

Despite the generality of our work, we report the execution time and required memory for HAT+CSI and Sup+CSI. The report is based on a machine with NVIDIA RTX 3090 on C10-5T experiments. HAT+CSI takes 28.68 hours while Sup+CSI runs for 18.41 hours, {which are slower than baselines.}
Contrastive learning and extensive data augmentation in CSI are the major reason for the slow execution time. However, if other more efficient OOD detection algorithms can replace CSI, the running time can be improved with the new OOD detection methods. 

As noted in Sec.~\ref{sec.training}, all the methods use the same backbone architecture with the same width and depth except for OWM and HyperNet for the reasons explained in the main paper. We report the number of parameters of each method required for inference after learning the last task. Sup and Sup+CSI uses a  very small number of parameters because Sup finds a sparse subnetwork for each task. Our methods HAT+CSI introduces 7.7K, 17.6K, 68.0K, 47.5K, 191.0K, and 109.0K parameters on M-5T, C10-5T, C100-10T, C100-20T, T-5T, and T-10T, respectively, at each task. Sup+CSI introduces 56.3K, 284.9K, 590.3K, 580.0K, 607.7K, and 605.1K parameters on the same experiments. The calibrated methods HAT+CSI+c and Sup+CSI+c introduce 2 parameters ($\alpha_k, \beta_k$) per task.

For HAT and HAT+CSI, the reported number of parameters is based on the network at full capacity. The hard attention masks consume 71.10, 86.31, 98.89, 99.71, 92.94, and 98.67\% of the total network capacity on average over 5 runs for HAT on M-5T, C10-5T, C100-10T, C100-20T, T-5, and T-10T, respectively. Similarly, 99.39, 99.56, 99.56, 50.68, 94.94, and 99.18\% of the total network capacity are used for HAT+CSI on the same datasets on average.

\section{Negative Societal Impacts}\label{apx:societalimpact}
The goal of continual learning is to learn a sequence of tasks incrementally. Like many machine learning algorithms, our proposed methods could be affected by bias in the input data as this work does not deal with fairness or bias in the data. A possible solution to mitigate the problem is to check bias in data before training.

\section{Forgetting Rate}\label{apx:forgetting}
We discuss forgetting rate (i.e., backward transfer)~\cite{Lopez2017gradient}, which is defined for task $t$ as
\begin{align}
    \mathcal{F}^{t} = \frac{1}{t-1} \sum_{k=1}^{t-1} \mathcal{A}_{k}^{\text{init}} - \mathcal{A}_{k}^{t},
\end{align}
where $\mathcal{A}_{k}^{\text{init}}$ is the classification accuracy of task $k$'s data after learning it for the first time and $\mathcal{A}_{k}^{t}$ is the accuracy of task $k$'s data after learning task $t$. We report the forgetting rate after learning the last task.

\begin{figure}[h!]
    \centering
    \includegraphics[width=4in]{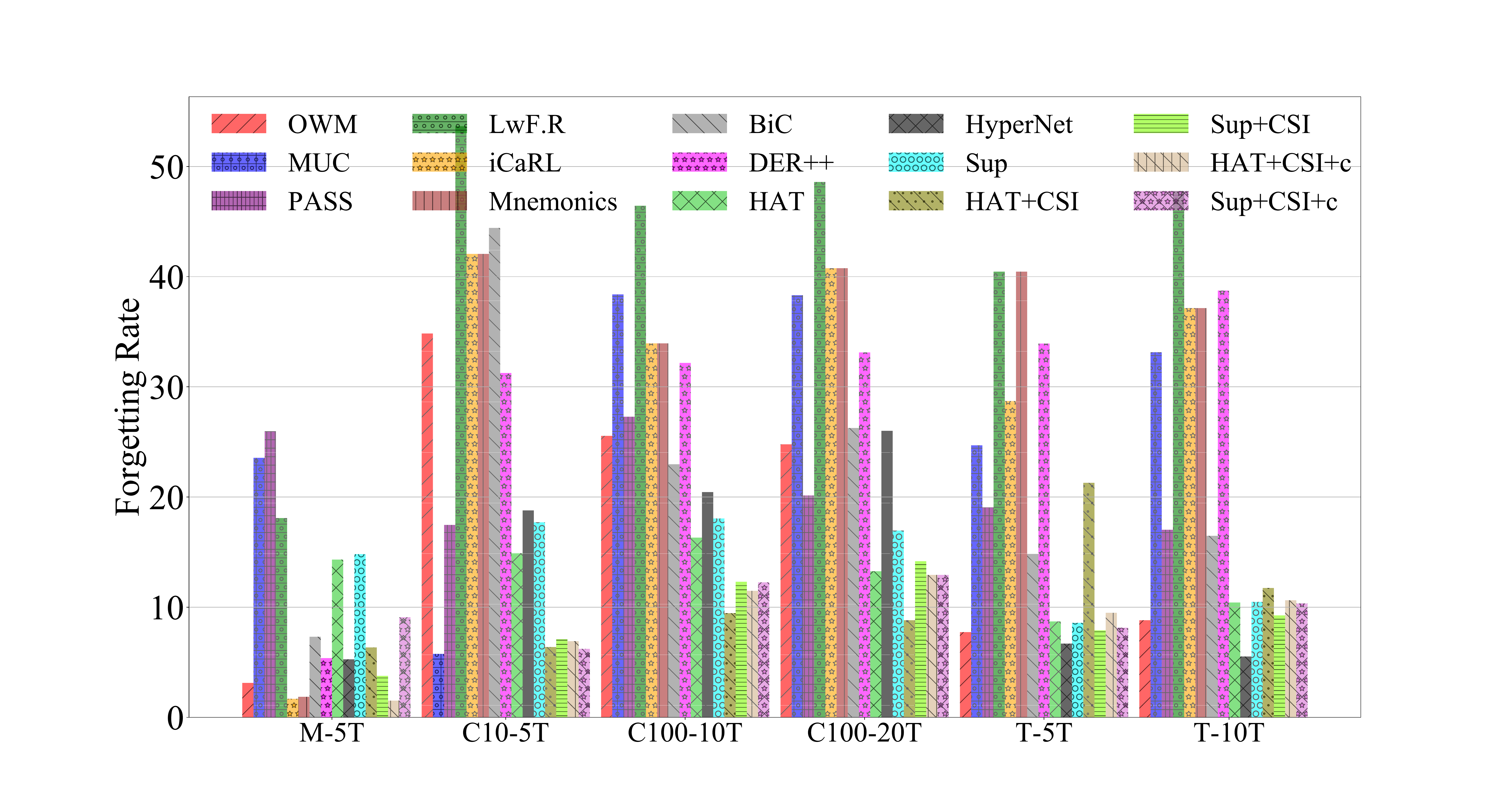}
    \caption{
    Average forgetting rate (\%). The lower the value, the better the method is on forgetting.
    }
    \label{fig:forget}
\end{figure}

Figure~\ref{fig:forget} shows the forgetting rates of each method. Some methods (e.g., OWM, iCaRL) experience less forgetting than the proposed methods HAT+CSI and Sup+CSI on M-5T. On this dataset, all the systems performed well. For instance, OWM and iCaRL achieve 95.8\% and 96.0\% accuracy while HAT+CSI and HAT+CSI+c achieve 94.4 and 96.9\% accuracy. As we have noted in the main paper, Sup+CSI and Sup+CSI+c achieve only 80.7 and 81.0 on M-5T although they have improved drastically from 70.1\% of the base method Sup.

OWM and HyperNet show lower forgetting rates than HAT+CSI+c and Sup+CSI+c on T-5T and T-10T. However, they are not able to adapt to new classes as OWM and HyperNet achieve the classification accuracy of only 10.0\% and 7.9\%, respectively, on T-5T and 8.6\% and 5.3\% on T-10T. HAT+CSI+c and Sup+CSI+c achieves 51.7\% and 49.2\%, respectively, on T-5T and 47.6\% and 46.2\% on T-10T.

In fact, the performance reduction (i.e., forgetting) in our proposed methods occurs not because the systems forget the previous task knowledge, but because the systems learn more classes and the classification naturally becomes harder. The continual learning mechanisms (HAT and Sup) used in the proposed methods experience little or no forgetting because they find independent subset of parameters for each task, and the learned parameters are not interfered during training. {\color{black}For the forgetting rate results in the TIL setting, refer to our earlier workshop paper \cite{kim2022continual}.}

\end{document}